\documentclass{article}

\usepackage{arxiv}

\usepackage[utf8]{inputenc} 
\usepackage[T1]{fontenc}    
\usepackage{hyperref}       
\usepackage{url}            
\usepackage{booktabs}       
\usepackage{amsfonts}       
\usepackage{nicefrac}       
\usepackage{microtype}      
\usepackage{lipsum}
\usepackage{graphicx}
\usepackage{times}
\usepackage{soul}
\usepackage{url}
\usepackage{amsmath}
\usepackage{amsthm}
\usepackage{booktabs}
\usepackage{algorithm}
\usepackage{algorithmic}
\usepackage{amsfonts} 

\newcommand*{\lmset}{\{\mskip-4mu\{}
\newcommand*{\rmset}{\}\mskip-4mu\}}

\newcommand{\set}[1]{\left\{ #1\right\}}

\newcommand{\mult}[1]{\operatorname{m}_{#1}}

\newcommand{\msetch}[2]{
\left.\mathchoice
  {\left(\kern-0.45em\binom{#1}{#2}\kern-0.45em\right)}
  {\big(\kern-0.10em\binom{\smash{#1}}{\smash{#2}}\kern-0.10em\big)}
  {\left(\kern-0.10em\binom{\smash{#1}}{\smash{#2}}\kern-0.10em\right)}
  {\left(\kern-0.10em\binom{\smash{#1}}{\smash{#2}}\kern-0.10em\right)}
\right.}

\newcommand{\dm}[2]{\inprod{#1}{#2}}

\newcommand{\LSC}{\mcL_{\operatorname{SC}}}

\newcommand{\bbI}{\mathbb{I}}

\newcommand{\bbN}{\mathbb{N}}

\newcommand{\bbS}{\mathbb{S}}

\newcommand{\R}{\mathbb{R}}

\newcommand{\mcB}{\mathcal{B}}

\newcommand{\mcL}{\mathcal{L}}

\newcommand{\mcY}{\mathcal{Y}}
\newcommand{\mcZ}{\mathcal{Z}}

\newcommand{\rmE}{\mathrm{E}}

\newcommand{\inprod}[2]{\langle #1 , #2 \rangle}

\newcommand\restr[2]{{
  \left.\kern-\nulldelimiterspace 
  #1 
  \vphantom{\big|} 
  \right|_{#2} 
  }}

\usepackage{multirow}
\usepackage{microtype}
\usepackage{graphicx}
\usepackage{subfigure}
\usepackage{booktabs} %

\usepackage{amsmath}
\usepackage{amssymb}
\usepackage{mathtools}
\usepackage{amsthm}
\usepackage{thmtools}
\usepackage{thm-restate}
\usepackage{placeins}

\theoremstyle{plain}

\newtheorem{lemma}{Lemma}

\theoremstyle{definition}
\newtheorem{definition}{Definition}

\theoremstyle{remark}
\newtheorem{remark}{Remark}

\usepackage{nccmath}
\renewcommand{\emptyset}{\phi}
\usepackage{algorithm}
\usepackage{algorithmic}

\usepackage{cleveref}
\crefname{table}{table}{table}
\Crefname{table}{table}{table}
\crefname{algorithm}{algorithm}{algorithm}
\Crefname{algorithm}{algorithm}{algorithm}
\Crefname{figure}{figure}{figure}
\crefname{figure}{figure}{figure}
\crefname{claim}{claim}{claim}
\Crefname{claim}{claim}{claim}
\crefname{lemma}{lemma}{lemma}
\Crefname{lemma}{lemma}{lemma}
\usepackage{pifont}
\usepackage{xcolor}
\usepackage{float}

\declaretheoremstyle[
    headfont=\bfseries, 
    bodyfont=\itshape
]{boldstyle}

\usepackage{pgfplots}
\pgfplotsset{compat=1.18} 
\usepackage{pgfplotstable} 
\usepackage{xstring} 
\usepackage{pgffor} 

\usepackage{tabularx}

\usepackage{svg}
\usepackage{amssymb}

\title{Global Pre-fixing, Local Adjusting: A Simple yet Effective Contrastive Strategy for Continual Learning}

\author{
 Jia Tang \\
 Nanjing University of Aeronautics and Astronautics\\
  \texttt{tangjia@nuaa.edu.cn} \\
   \And
 Xinrui Wang \\
  Nanjing University of Aeronautics and Astronautics\\
  \texttt{wangxinrui@nuaa.edu.cn} \\
  \And
 Songcan Chen \thanks{Corresponding author: s.chen@nuaa.edu.cn}  \\
  Nanjing University of Aeronautics and Astronautics\\
  \texttt{s.chen@nuaa.edu.cn} \\
}

\begin{document}
\maketitle

\begin{abstract}
Continual learning (CL) involves acquiring and accumulating knowledge from evolving tasks while alleviating catastrophic forgetting. Recently, leveraging contrastive loss to construct more transferable and less forgetful representations has been a promising direction in CL. Despite advancements, their performance is still limited due to confusion arising from both inter-task and intra-task features. To address the problem, we propose a simple yet effective contrastive strategy named \textbf{G}lobal \textbf{P}re-fixing, \textbf{L}ocal \textbf{A}djusting for \textbf{S}upervised \textbf{C}ontrastive learning (GPLASC). Specifically, to avoid task-level confusion, we divide the entire unit hypersphere of representations into non-overlapping regions, with the centers of the regions forming an inter-task pre-fixed \textbf{E}quiangular \textbf{T}ight \textbf{F}rame (ETF). Meanwhile, for individual tasks, our method helps regulate the feature structure and form intra-task adjustable ETFs within their respective allocated regions. As a result, our method \textit{simultaneously} ensures discriminative feature structures both between tasks and within tasks and can be seamlessly integrated into any existing contrastive continual learning framework. Extensive experiments validate its effectiveness.
\end{abstract}


\section{Introduction}
\label{sec:intro}

Continual learning (CL) learns from a sequence of tasks while addressing the Catastrophic Forgetting (CF) problem \cite{catastrophic}. A key scenario in CL is Class Incremental Learning (CIL), where the model must not only maintain intra-task discriminability but also distinguish between classes from different tasks without access to task identifiers during inference. This setup requires the model to handle classes it has never encountered together during training, making CIL particularly challenging.

So far, numerous approaches have been proposed, predominantly learned by minimizing cross-entropy loss (CE). However, a recent study Co2L\cite{cha2021co2l} and its variants\cite{cclis,wen2024provable} leverage supervised contrastive loss (SupCon)\cite{SupCon} to learn features with less forgetting compared to CE-trained ones. We argue that their success can be due to learning implicitly
\textit{intra-task} features\footnote{Unless otherwise specified, features in our work are normalized and exist on the surface of a hypersphere.} geometry from SupCon in the form of a simplex \textbf{E}quiangular \textbf{T}ight \textbf{F}rame(ETF)\cite{graf2021dissecting}.
Such structure offers better transferability and discriminability\cite{etf_transfer,neural_collapse}. However, as shown in \Cref{fig:local_supervised_contrastive_loss_img} (left), due to the sequential arrival of tasks and the (global) invisibility of inter-task data, the SupCon cannot spontaneously form a \textit{task-level} well-separated feature structure in CL without explicit separability requirement, thus degrading the whole learning performance of these methods.

In order to achieve \textit{inter-task} separable representation, some methods have been proposed\cite{2021gpm,shiprospective}.
They optimize the current task's features within a subspace that does not intersect with those of past tasks to minimize inter-task interference.
However, as the number of learned tasks increases, the optimization subspace for the current task will be severely restricted. 
On the one hand, directly applying such idea in Co2l cannot ensure the formation of the intra-task ETF when optimizing within constrained subspace, resulting in intra-task feature confusion.
On the other hand, such methods overlook the task-level well-separated feature structure, limiting their inter-task performance.

Building on the above analysis, we follow Co2L but differ in two key motivations:

\noindent\textbf{For inter-task features: global pre-fixing.}
Obtaining an \textit{inter-task} discriminative features through directly optimization is relatively difficult in CL.
The underlying reasons are: (1) We cannot replay sufficient previous task data due to the memory constraint in the context of CL.
(2) The sequential arrival of tasks makes it challenging to achieve global inter-task discrimination by focusing solely on the current task's optimization.
Facing with these, intuitively, if we could globally pre-allocate non-overlapping feature regions for different tasks from the beginning, the model would be equipped with task-level discriminability inherently.

\begin{figure*}[!ht]
	\centering
    \includegraphics[width=0.9\textwidth]{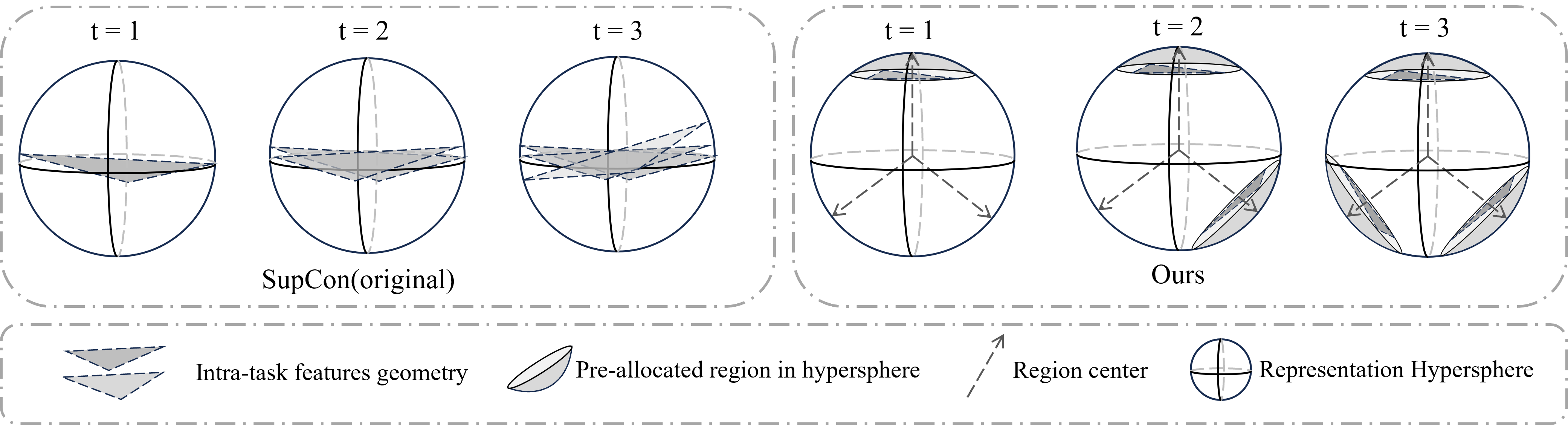} 
	\caption{
    The difference between our method and SupCon in CL. 
    \textbf{(left):} Due to the sequential arrival of tasks and the (global) invisibility of inter-task data, the SupCon cannot spontaneously form a \textit{task-level} well-separated feature structure in CL.
    \textbf{(right):} 
We pre-divide the unit representation hypersphere into non-overlapping regions according to the pre-estimated maximum task counts, with the center of the regions forming a task-level pre-fixed ETF. The proposed loss helps adjust feature structure and form an intra-task ETF in the allocated region.}  
	\label{fig:local_supervised_contrastive_loss_img}
\end{figure*}

\noindent\textbf{For intra-task features: local adjusting.} Task-level feature discriminability is not enough yet for CL. Intra-task separability is also crucial for performance. Considering each task separately holds a pre-allocated feature region, we can independently optimize to learn the class-level features of the current task within its own region, intuitively able to secure locally-optimal discrimination. 


In addition, it is worth to point out that such \textit{global pre-fixing and local adjusting} strategy can be applied to any training paradigm, as long as the intra-task optimization is appropriately adapted. In our work, we adopt SupCon for intra-task features optimization because it can better construct feature structure, leading to improved discriminability and reduced forgetting\cite{cha2021co2l}. Furthermore, based on the above motivations, we design a new contrastive loss for each task, named \textbf{R}egion \textbf{R}estricted \textbf{S}upervised \textbf{C}ontrastive \textbf{L}oss (R2SCL), to follow the contrastive paradigm and form an intra-task ETF in a specific region. Furthermore, we propose a simple yet effective contrastive strategy and learning method
for CL: \textit{\textbf{G}lobal \textbf{P}re-fixing, \textbf{L}ocal \textbf{A}djusting} for \textbf{S}upervised \textbf{C}ontrastive learning(GPLASC). As shown in \Cref{fig:local_supervised_contrastive_loss_img}, compared to SupCon, our method achieves a separable feature structure both between tasks and within tasks as tasks arrive continually. 

There are some works similar to ours, such as PRL\cite{shiprospective} and LODE\cite{liang2023loss}.
PRL try to separate features from different classes in continual learning. However, it does not explore how to incrementally construct discriminative geometric structure during the separation which limits their performance. LODE’s limitations are its reliance on CE loss and one-hot labels to provide separation information.
Although our work can be interpreted from a decoupling perspective as LODE and PRL, they differ in their nature and present distinct challenges. Our work focuses on incrementally building discriminative representation geometry under continual learning, and we propose a more general strategy.





Specifically, for \textit{local adjusting}, we first theoretically analyze the key factors that influence the individual feature regions of SupCon by revisiting its lower bound\cite{graf2021dissecting}. Our theoretical result reveals that by adjusting the intra-task feature similarity and task prototypes, we can restrict the individual feature region on the unit hypersphere and preserve intra-task ETF structure in the region. Based on our theoretical work, we design a new contrastive loss for each task to adjust features in the specific region.
For \textit{global pre-fixing}, we pre-allocate the task regions with their centers forming a fixed inter-task ETF. Furthermore, 
the proposed loss helps constrain the features of different tasks to their allocated regions. 

On the other hand, as tasks arrive continually, learned features will gradually drift away from their pre-fixed regions\cite{sdc}. Existing contrastive continual learning methods\cite{cha2021co2l,wen2024provable} usually adopt an instance relation distillation to prevent relative drift between instances. Our experimental results show that introducing and optimizing an additional feature-level Mean Squared Error (MSE) loss can mitigate the drift more effectively, leading to better performance.

In the following, we present our contributions:
\begin{itemize}

\item To tackle the feature confusion of existing contrastive continual learning methods, we propose a strategy that combines global prefixing with local adjustment to decouple the intra-task and inter-task feature learning. For the former, we pre-allocate non-overlapping regions respectively for different tasks to avoid inter-task feature confusion. For the latter, we optimize the features localizing in pre-allocated region to mitigate intra-task feature confusion.

\item To deconfuse intra-task class-level features according to the above strategy, inspired by SupCon that can implicitly learn the ideal discriminative feature structure of ETF, we propose a new contrastive loss. By optimizing the loss, we can form an intra-task ETF for each task in its pre-allocated region. We provide a theoretical explanation for this.

\item We evaluate our proposed method with various contrastive continual learning methods to confirm its superiority. Our method achieves a consistent improvement in CIL performance across all datasets with different buffer sizes.

\end{itemize}

\section{Preliminary}

\label{sec:preliminary}
\subsection{Problem Setup}
Continual learning aims to learn from a sequence of tasks, indexed by $t=1,2,...,T$, where each task has a dataset denoted as $D^t=\{(x_i^t,y_i^t)\}_{i=1}^{N^t}$. $N^t$ is the size of dataset in task $t$. $x_i^t$ is the instance with the label $y_i^t \in \mathcal{Y}^t$ in task $t$, $\mathcal{Y}^t$ is the corresponding label space where $\mathcal{Y}^t \cap \mathcal{Y}^{t^\prime} = \oslash$ for $t \ne t^\prime$. We denote $|\mathcal{Y}^t|$ as the class number of the task $t$.
We use the superscript $t$ to represent the task ID. 
In traditional CL, we train a network $f_{\theta}=f_c \circ f_e $  with parameters $\theta$, including a backbone $f_e$ and a classifier $f_c$. For a specific task, $f_e$ maps the input $\{x_1^t,x_2^t,\ldots,x_{N^t}^t\}$ to representation embeddings $\mathcal{Z}^t=\{z_1^t,z_2^t,\ldots,z_{N^t}^t\},\text{ i.e. }z_i^t = f_e(x_i^t)$. 
During training, we use a memory buffer $D_M^{t-1}$ to store a small subset of data from the previous tasks. In the \( t \)-th task, we train on $ D^t \cup D_M^{t-1}$. 
There are two important settings in CL: Class-Incremental Learning (CIL) and Task-Incremental Learning (TIL). Their main difference lies in the evaluation phase that TIL has access to task ID $t$, while CIL does not. 
Our method is designed for CIL.


\subsection{Contrastive Continual Learning}
Contrastive continual learning leverages supervised contrastive loss (Supcon)\cite{SupCon} to achieve less forgettable features\cite{cha2021co2l}. Specifically, it focuses on the encoder $f_e$. Specifically, for a specific task,  let $Z=(z_1,\ldots,z_N)$ be $N$ normalized features with corresponding labels $Y=(y_1,\ldots,y_N)$, we aim to optimize the total loss 
    $ \LSC(Z;Y) =\sum\nolimits_{B \in \mcB} \mathcal{L}_{SupCon}(Z; Y, B)$
    , with $\mcB$ representing a batch set for current task.
For a specific batch $B$ with batchsize $b$, we calculate $\mathcal{L}_{SupCon}(Z; Y, B)$,  which is defined \footnote{
Following \cite{graf2021dissecting}, we omit the term  $\tau$  for simplicity and set $\frac{1}{|B_{y_i}|-1}\bbI_{\set{|B_{y_i}|>1}}=0$ when $|B_{y_i}|=1$ for notational reasons.}
as:
\begin{equation}
\nonumber
    -\sum\limits_{i \in B}
    \frac{\bbI_{\{|B_{y_i}|>1\}}}{|B_{y_i}|-1}\hspace{-0.1cm}
    \sum\limits_{j \in B_{y_i}\setminus \lmset i \rmset} \hspace{-0.2cm}
    \!\log
    \left(
        \frac
        {\exp\big(\dm{z_i}{z_j}\big)}
        {
            \sum\limits_{k \in B \setminus \lmset i \rmset}\hspace{-0.2cm}
            \exp\big(\dm{z_i}{z_k}\big)
        }
    \right)
    \label{def:llsc}
    \end{equation}
 where $B_{y_i} = \lmset j \in B: y_{j} = y_i \rmset$ denotes the multi-set of indices in a batch $B \in \mcB$ with label equal to $y_i$. $\bbI(\cdot)$ is the indicator function.

\subsection{Simplex Equiangular Tight Frames}

Equiangular tight frames (ETFs) are versatile in many fields due to their unique properties of equiangularity and tightness. These properties make ETFs particularly useful in applications requiring minimal correlation and optimal redundancy, such as communication systems, quantum information, signal processing and image processing.
    In \Cref{def:simplex}, we introduce a special ETF\cite{fickus2018equiangular},  
    $\rho$-sphere-inscribed regular simplex\cite{graf2021dissecting}
    , which will be widely used in our proof.
\begin{definition}[\textbf{$\rho$-Sphere-inscribed regular simplex}] 
\label{def:simplex}
    Let $\bbN$ denote the set of natural numbers, $h,K \in \bbN$ with $K\le h+1$.
    We say that $\zeta_1, \dots, \zeta_K \in \R^h$ form the vertices of a regular simplex inscribed in the hypersphere of radius $\rho>0$, if and only if the following conditions\footnote{To facilitate references, we assign a name to each condition (see parentheses)} hold:
 \begin{itemize}
        \item \label{def:simplex:s1} $\sum_{i \in [K]} \zeta_i = 0$ \textit{(Centroid Condition)}
        \item \label{def:simplex:s2} $\| \zeta_i \| = \rho$\, for $i \in [K]$ \textit{(Radius Condition)}
        \item \label{def:simplex:s3} $\exists d \in \R: d = \inprod{\zeta_i}{\zeta_j}$\, for $1 \leq i < j \leq K$ \textit{(Equiangular Condition)}
    \end{itemize}
\end{definition}

\section{Methods}
\label{sec:analysis}

\textbf{Challenges.} In CIL, despite intra-task discriminability, the model needs to separate classes from different tasks which have never been seen together before. Based on this, we decompose the challenges of CIL into three parts: \textit{intra-task feature confusion}, \textit{inter-task feature confusion} and \textit{forgetting}.

\noindent\textit{Intra-task feature confusion:} 
On the one hand, to minimize forgetting during learning, various methods optimize the current task in a subspace that does not intersect with past tasks, which may potentially affect the plasticity of the current task as task number increases. On the other hand, complex datasets and the model's poor ability to learn discriminative representations also lead to feature confusion.




\noindent\textit{Inter-task feature confusion:}
Various existing works confuse \textit{ inter-task confusion} and \textit{forgetting} in CIL. They view task confusion as one of the reasons for forgetting\cite{soutif2021importance}. However, the \textit{Infeasibility Theorem}\cite{nipstf_cf} has proved that even the discriminative model without forgetting cannot reach the optimal CIL performance due to task confusion, suggesting their distinct roles. Such a problem arises from dynamic task distributions and the invisibility of inter-task data, which poses a challenge to the \textit{task-level} discriminative learning.

\noindent\textit{Forgetting:} The model experiences a performance decline on earlier tasks. From the perspective of representation, the learned features will drift away from their original position as tasks arrive continually\cite{sdc}.

\noindent\textbf{How to Construct More Effective Representations in CIL: A Simple Overview of Our Method.}
According to \cite{kim2022theoretical}, the CIL performance is bounded by the performance of both the within-task prediction (WP) and task-id prediction (TP). In other words, a discriminative feature in CL needs to consider both the intra-task and the inter-task performance.
Therefore, we decouple the training of features to two aspects: \textit{global pre-fixing}, \textit{local adjusting}.
A simple illustration of our method is shown in \Cref{fig:local_supervised_contrastive_loss_img}.
Specifically, we enhance inter-task separability by pre-allocating distinct, non-overlapping regions for different tasks, while intra-task discriminability is obtained by adjusting contrastive features in the task-specific region.
We detail this in the following subsection:
\begin{itemize}
   \item \Cref{controable_contrastive_learning}: Our theoretical work. We explore the key factors that control the final feature structure of SupCon, which provides the theoretical foundation for the \textit{local adjusting}.

    \item \Cref{sec:scaffold} (the first part) : Implementation of \textit{local adjusting}. Based on our theoretical results, we design and optimize a contrastive loss to adjust the feature to the resided region for a specific task.
    \item \Cref{sec:scaffold} (the second part) : Implementation of \textit{global pre-fixing}. We pre-allocate the non-overlapping feature regions for different tasks and then leverage the proposed loss to adjust features for each task.
    

    \item \Cref{sec:simple_distallation}: Implementation of more effective distillation in contrastive continual learning. We introduce an additional feature-level MSE loss to better prevent the feature drift of past tasks.

\end{itemize}

\subsection{Revisiting Feature Structure in SupCon from a Theoretical Perspective}
\label{controable_contrastive_learning}

In this section, we discuss the features formed by the SupCon and try to identify the keys that control the feature structure. 
All of the following analyses are focused on training for a specific task. 
For simplicity, we omit the superscript $t$.
We start by considering the following lemma:

\begin{restatable}{lem}{lower@bound@contrastive}
    \label{lemma:lower_bound_contrastive}
    Let $Z=(z_1,\ldots,z_N)$ be $N$ normalized features with corresponding labels $Y=(y_1,\ldots,y_N)$. If the label configuration $Y$ is balanced, it holds that 
    \begin{align}
    \nonumber
        &\LSC(Z;Y) \notag
        \ge 
        \sum_{l = 2}^{b} l M_l                           
            \log 
            ( 
                l - 1 + (b-l)
                \exp ( 
                    \frac {1} {M_l}
                     \\ \notag
                        &\underbrace{
                            ( - \frac{1}{|B_{y}|\left(|B_{y}|-1\right)}\sum_{y\in \mathcal{Y}}{\sum_{B\in 
                            \mathcal{B}_{y,l}}{
                                \sum_{i\in B_{y}}\sum_{j\in B_{y}\setminus\{\{i\}\}}\langle z_{i},z_{j}\rangle 
                            }}
                        }_{\text{attraction term}}\\
                        &+
                        \underbrace{
                            ( 
                            \frac{|\mathcal{B}_{y,l}|}{N^2}\frac{|\mathcal{Y}|^2}{|\mathcal{Y}|-1}
                            \sum_{y\in\mathcal{Y}}\sum_{ \substack{n\in[N]\\y_n=y}}
                            \sum_{y_m\neq y}\langle z_n,z_m \rangle 
                            ) 
                        }_{\text{repulsion term}}
                    )
                )
            )\nonumber
        \enspace,
    \end{align}
      where $M_l =\sum_ {y \in \mcY} |\mcB_{y,l}|$. 
\end{restatable}

Following the notations used in \cite{graf2021dissecting}, $\mcB_{y,l}$ represents the specific batch that contains $l$ samples with label $y$. For each batch $B$ with batchsize $b$. In fact, \Cref{lemma:lower_bound_contrastive} is an intermediate result from \cite{graf2021dissecting}.
For a more detailed proof, see \cite{graf2021dissecting}.

\Cref{lemma:lower_bound_contrastive} decomposes the SupCon into two components: the \textit{attraction term}, which pulls positive pairs closer, and the \textit{repulsion term}, which separates negative pairs. 
Minimizing the \textit{attraction term} will cause features within the same class to collapse into a single point and minimizing the \textit{ repulsion term} will lead to the dispersion of different class features across the entire hypersphere.
If the class labels within a task are balanced, the representations of each class will finally collapse to the vertices of a $\rho_{\mcZ}$-sphere-inscribed regular simplex \cite{graf2021dissecting}. The inscribed hypersphere has its center at the origin and a radius of $\rho_{\mcZ}=1$.
Such a feature structure ensures that the separation is uniform, i.e., no class is closer or farther compared to others.

Notice the close relationship between the feature structure and the \textit{repulsion term}, which represents the similarity of intra-task features from different classes, we try to explore and demonstrate such relationship theoretically. We denote the lower bound of the similarity by $k$, i.e. $ \langle z_i,z_j \rangle \geq k $, $z_i,z_j$ are features of different classes. Obviously, in \Cref{lemma:lower_bound_contrastive}, $k = -1$.
Specifically, we show the following theorem:

\begin{restatable}[\textbf{Lower Bound and Final Feature Geometry for Supervised Contrastive Loss with Inter-Class Feature Similarity Constraints}]{thm}{localsupcontheorem}
    \label{lemma:local_supcon_theorem}
    
    $\forall$ $ i,j$  such that $ y_i \neq y_j$,   
    $\left< z_i,z_j \right> \ge k $ , \( h \geq  |\mcY|  - 1\). When \( -\frac{1}{|\mcY|-1} \leq k \leq 1 \),
    the supervised contrasive loss $\LSC(Z;Y)$ is bounded from below by
    \[
    \begin{array}{c}
    {{\cal L}_{{\rm{SC}}}}(Z;Y) \ge \sum\limits_{l = 2}^b l {M_l}\log (l - 1 + (b - l)\exp (
     - (1-k){\rho_{\mcZ}}^2
             ))
    \end{array}\]
    where equality is attained if and only if there are $\zeta_1, \dots, \zeta_{|\mcY|} \in \R^h$ such that the following conditions hold:
    
       \begin{itemize}
       \item $\forall n \in [N]: z_n = \zeta_{y_n}$.
        
        \item for $z_i, z_j$ belong to different classes, $\left< z_i,z_j \right> = k $.
        \item $\{\zeta_y\}_{y\in \mcY}$ form a $\rho$-sphere-inscribed regular simplex, with center at  $\frac{1}{N}\sum_i z_i$  and  radius of $\rho= \sqrt{(1-\frac{1}{|\mcY|})(1-k)}$ .
    \end{itemize}
\end{restatable}

We prove the above theorem in \Cref{appendix:proofs}. In \Cref{sec:k_annother}, we discuss the corresponding feature geometry and lower bound when $\ -1\leq k < -\frac{1}{|\mcY|-1} $. Furthermore, we prove that the final theoretical results
\footnote{According to \cite{graf2021dissecting}, optimizing SupCon leads to features of different classes forming a $\rho$-sphere-inscribed regular simplex, $\rho=1$ .}
about SupCon in \cite{graf2021dissecting} is a special case of \Cref{lemma:local_supcon_theorem}.
In fact, there is an implicit lower bound constraint on the similarity between negative sample pairs within a task in \Cref{lemma:lower_bound_contrastive}, i.e. $\left< z_i,z_j \right> \geq -1 $. We discuss a more general case.

As \Cref{lemma:local_supcon_theorem} shows, the lower bound of the inter-class similarity determines the radius and position of the final inscribed hypersphere of the intra-task ETF. Specifically, when \( -\frac{1}{|\mcY|-1} \leq k \leq 1 \), the radius of the hypersphere is reduced from $1$ to $ \sqrt{(1-\frac{1}{|\mcY|})(1-k)}$, indicating a smaller region as k increases. Moreover, the center shifts from the origin to $\frac{1}{N}\sum_i z_i$. The above conditions suggest that features no longer spread across the entire hypersphere, but instead focus on a restricted region of the hypersphere. Meanwhile, for a single task, the features will eventually form a scaled ETF, which is desirable because it ensures maximal class-level separation in the task-specific region. 
For clarity, we illustrate the main idea of the \Cref{lemma:local_supcon_theorem} in \Cref{fig:single}.

Note that the inscribed hypersphere of the final ETF is different from the original hypersphere when faced with inter-class feature similarity constraints. For clarity, we differentiate between them by referring to the \textit{final inscribed hypersphere} and the \textit{original hypersphere}. Obviously, when \( k = -\frac{1}{|\mcY|-1}\), $\rho= 1$. The \textit{final inscribed hypersphere} and \textit{original hypersphere} coincide.

\begin{figure}[H]
	\centering
    \includegraphics[width=0.5\linewidth]{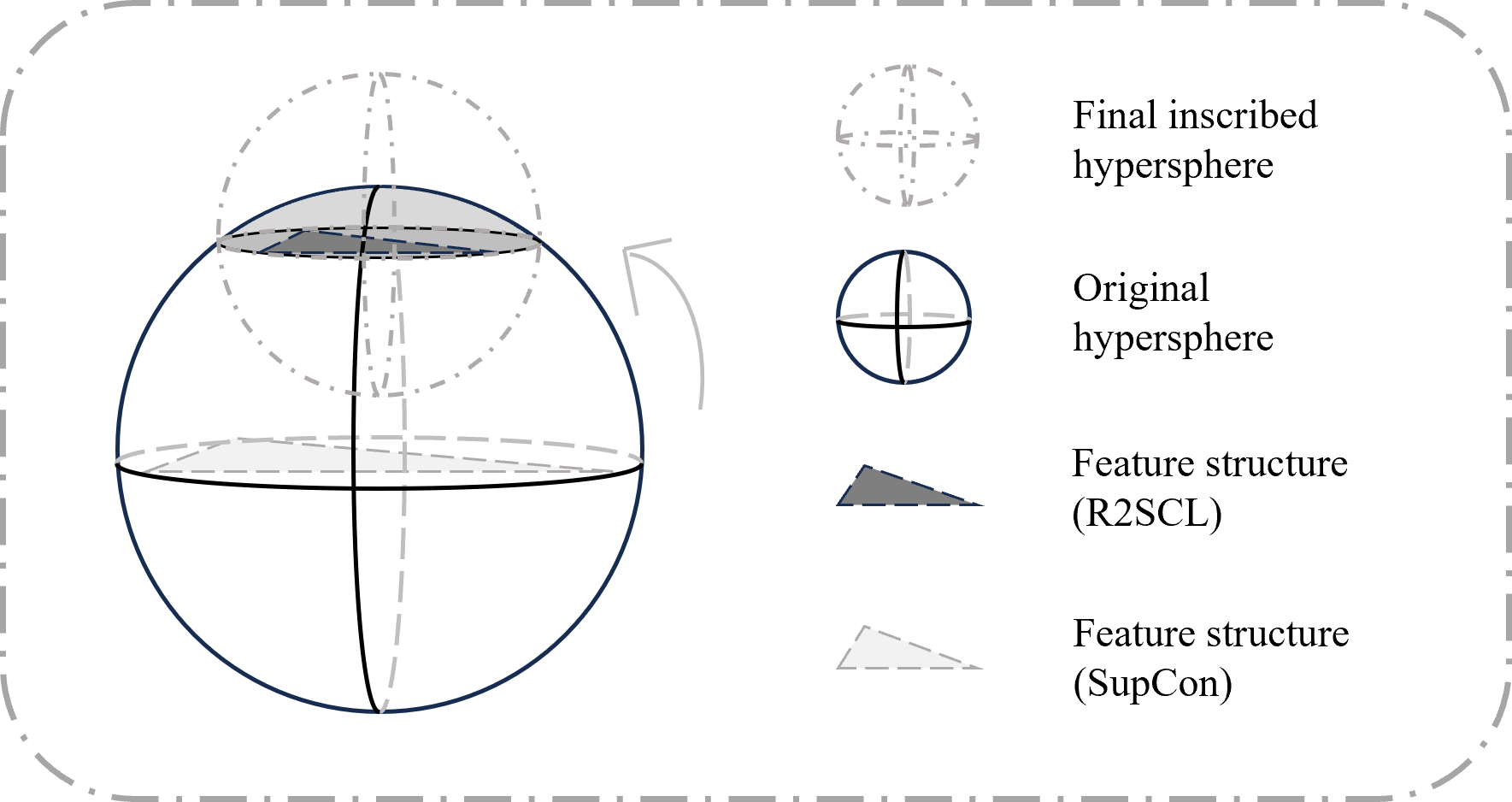}
	\caption{Illustration of \Cref{lemma:local_supcon_theorem}. 
    Varying the lower bound of the inter-class similarity will change the radius and position of the final inscribed hypersphere of the intra-task ETF.}
	\label{fig:single}
\end{figure}

\subsection{Global Pre-fixing, Local Adjusting for Supervised Contrastive Learning}
\label{sec:scaffold}

\Cref{lemma:local_supcon_theorem} indicates that constraints on inter-class similarity can result in scaling the range of the final feature structure on the hypersphere. This provides an insight for our algorithm design: we can divide the entire unit hypersphere of representation into non-overlapping fixed regions for different tasks to obtain inter-task separability. Meanwhile, we preserve the contrastive learning paradigm to obtain a discriminative intra-task structure within its allocated region. Next, we discuss our approach from \textit{local adjusting} and \textit{global pre-fixing} learning strategy.

\noindent\textbf{Local Adjusting:}
Local Adjusting aims to perform contrastive learning in a specific region. Based on \Cref{lemma:local_supcon_theorem}, we can set a threshold to constrain the intra-task similarity. A simple form would be:
\begin{equation}
   \begin{aligned}
    \label{def:naive_lsc}
    \mathcal{L}_{range}(Z;B)\triangleq 
    &\mathcal{L}_{SupCon}(Z; Y, B)\\
   +&\lambda_{range} \cdot \sum_{y_i \neq y_j}\max(0, k-\dm{z_i}{z_j})
    \end{aligned} 
\end{equation}
\textit{k} is the threshold to control the feature structure of contrastive learning by constraining the inter-class similarity. However, due to the random sampling of data in each batch, directly optimizing \Cref{def:naive_lsc} cannot restrict our features to a \textit{specific} region as expected. 
Note that the final inscribed hypersphere is determined by \textit{the radius} ($\rho= \sqrt{(1-\frac{1}{|\mcY|})(1-k)}$ ) and \textit{the center} ($\frac{1}{N}\sum_i z_i$). 
Specifically, the former determines the range of the distribution region, which is controlled by our constraints on the similarity in \Cref{def:naive_lsc}. The latter is naturally defined as the \textbf{task prototype} as shown in \Cref{def:task_prototype}, determining the position of the region.

\begin{definition}[Task Prototype] For task $t$, let $ Z^t=\{z_1^t,z_2^t,\ldots,z_N^t\}, z_i^t = f_e(x_i^t)$. We denote $P^t$ as its prototype:
\label{def:task_prototype}
    \begin{equation}
    \nonumber
        P^t \triangleq \frac{1}{N} \sum\nolimits_{i} z_i^t 
    \end{equation}
    
\end{definition}
Our method aims to constrain all task features within a fixed region of the representation space. In this context, the center of the hypersphere determines where the features will ultimately reside. To ensure that this location aligns with a predefined and desired position, we introduce a constraint on $P^t$ and $P_{fix}^t$, where $P_{fix}^t$ determines the position of the pre-allocated region.
The difference between them is that $P^t$ is the real prototype learned from the current task, while $P_{fix}^t$ is the desired prototype location. 
Building upon this, we introduce an additional MSE loss between the task prototype and the pre-fixed region center $P_{fix}^t$. The benefit of minimizing the distance between them is that it allows us to control the position of each task's feature distribution, which helps reduce inter-task interference and align different task prototypes $P^t$ to different pre-allocated centers $P_{fix}^t$ with maximal separation to overcome task-level confusion.

\begin{equation}
   \begin{aligned}
    \label{def:position_loss}
    \mathcal{L}_{position}(Z;B,P_{fix}^t)\triangleq
    \lambda_{position} \cdot \mathcal{L}_{mse}(P^t,P_{fix}^t)
    \end{aligned} 
\end{equation}

Based on \Cref{def:position_loss} and \Cref{def:naive_lsc}, we propose the \textbf{R}egion \textbf{R}estricted \textbf{S}upervised \textbf{C}ontrastive \textbf{L}oss (R2SCL) that can adjust contrastive feature in a task-specific region centered on $P_{fix}^t$.

\begin{equation}
   \begin{aligned}
    \label{def:prototype_loss}
    \mathcal{L}_{R2SCL}(Z;B,P_{fix}^t)=
    &\mathcal{L}_{range}(Z;B)\\
    +&\mathcal{L}_{position}(Z;B,P_{fix}^t)
    \nonumber
    \end{aligned} 
\end{equation}
\begin{figure*}[htbp]
	\centering
	\subfigure[no threshold]{
		\begin{minipage}[b]{0.13\textwidth}
			\includegraphics[width=1\textwidth]{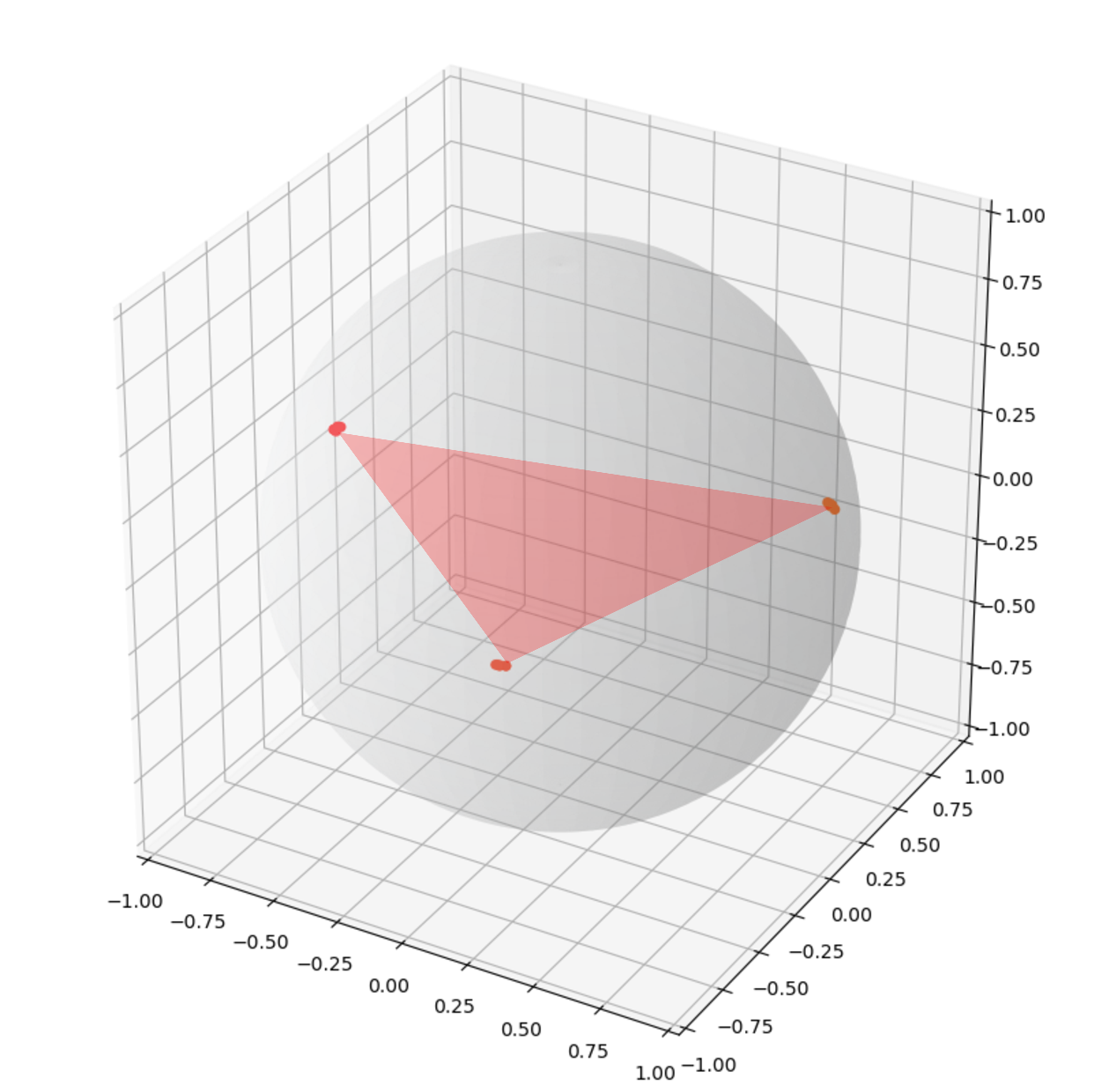}
		\end{minipage}
		\label{fig:different_lamda_a}
	}
    \subfigure[threshold=0.3]{
    	\begin{minipage}[b]{0.13\textwidth}
   		\includegraphics[width=1\textwidth]{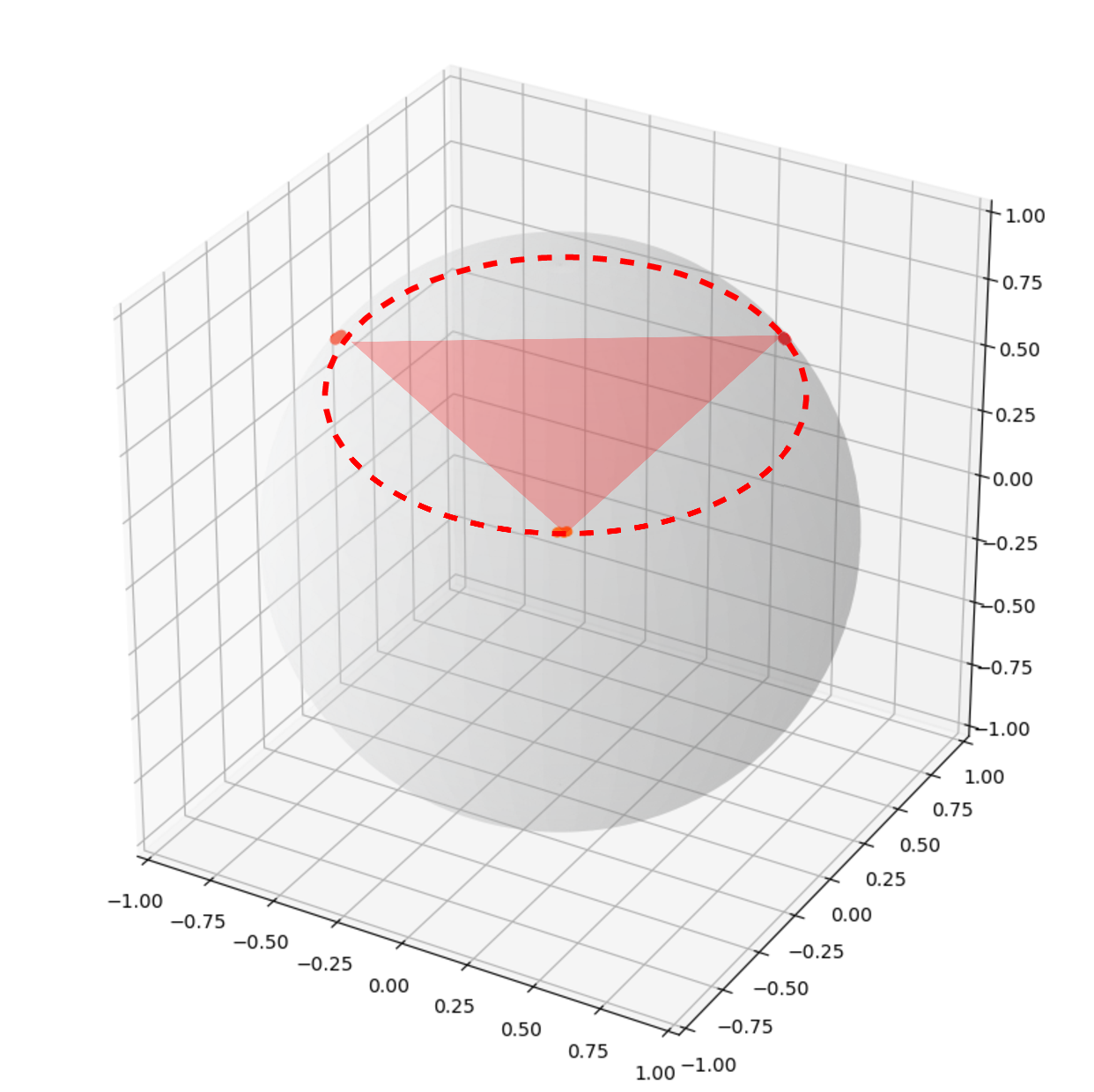}
    	\end{minipage}
	\label{fig:different_lamda_b}
    }
    \subfigure[threshold=0.7]{
    	\begin{minipage}[b]{0.13\textwidth}
   		\includegraphics[width=1\textwidth]{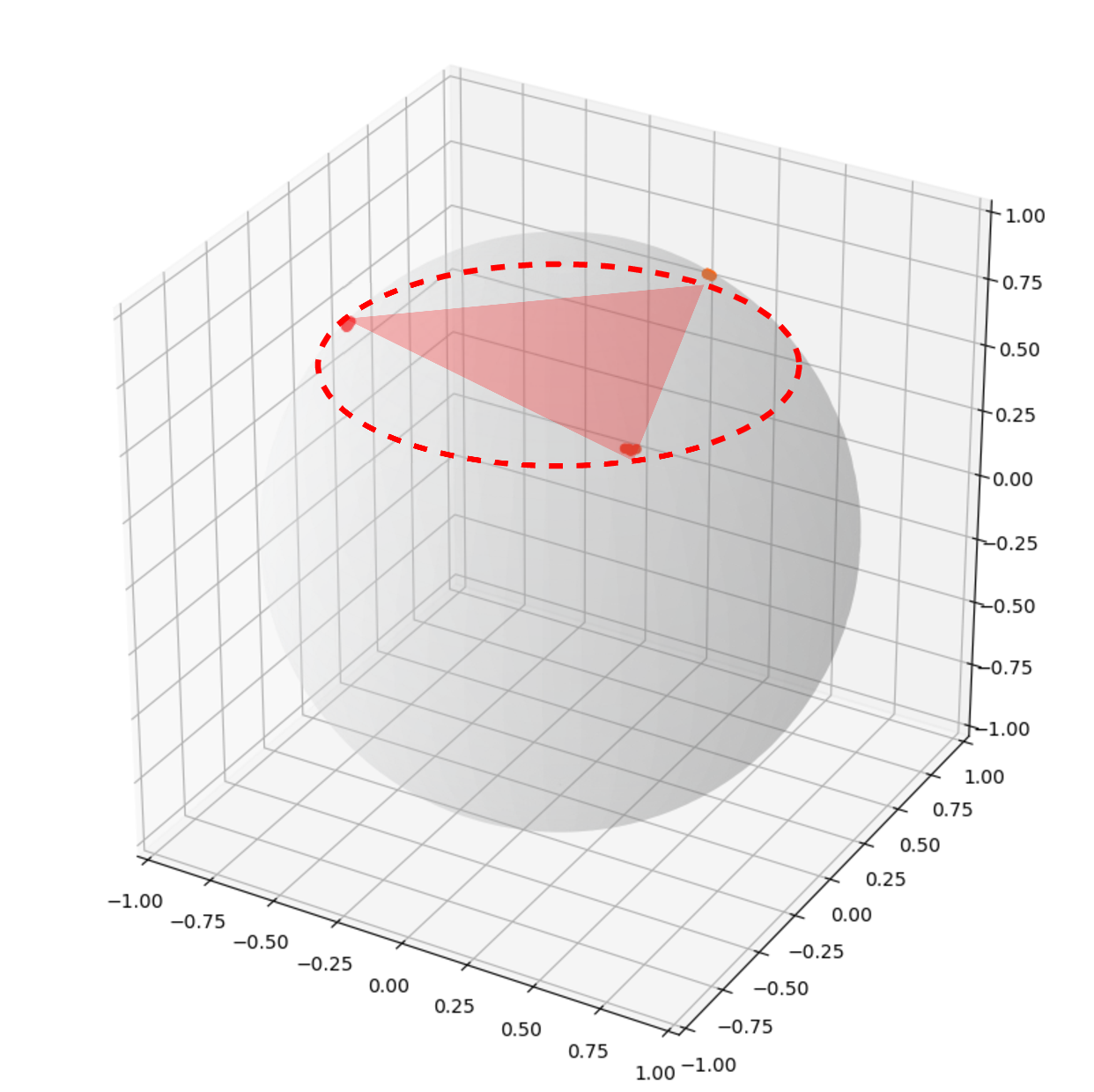}
    	\end{minipage}
	\label{fig:different_lamda_c}
    }
    \subfigure[threshold=0.9]{
    	\begin{minipage}[b]{0.13\textwidth}
   		\includegraphics[width=1\textwidth]{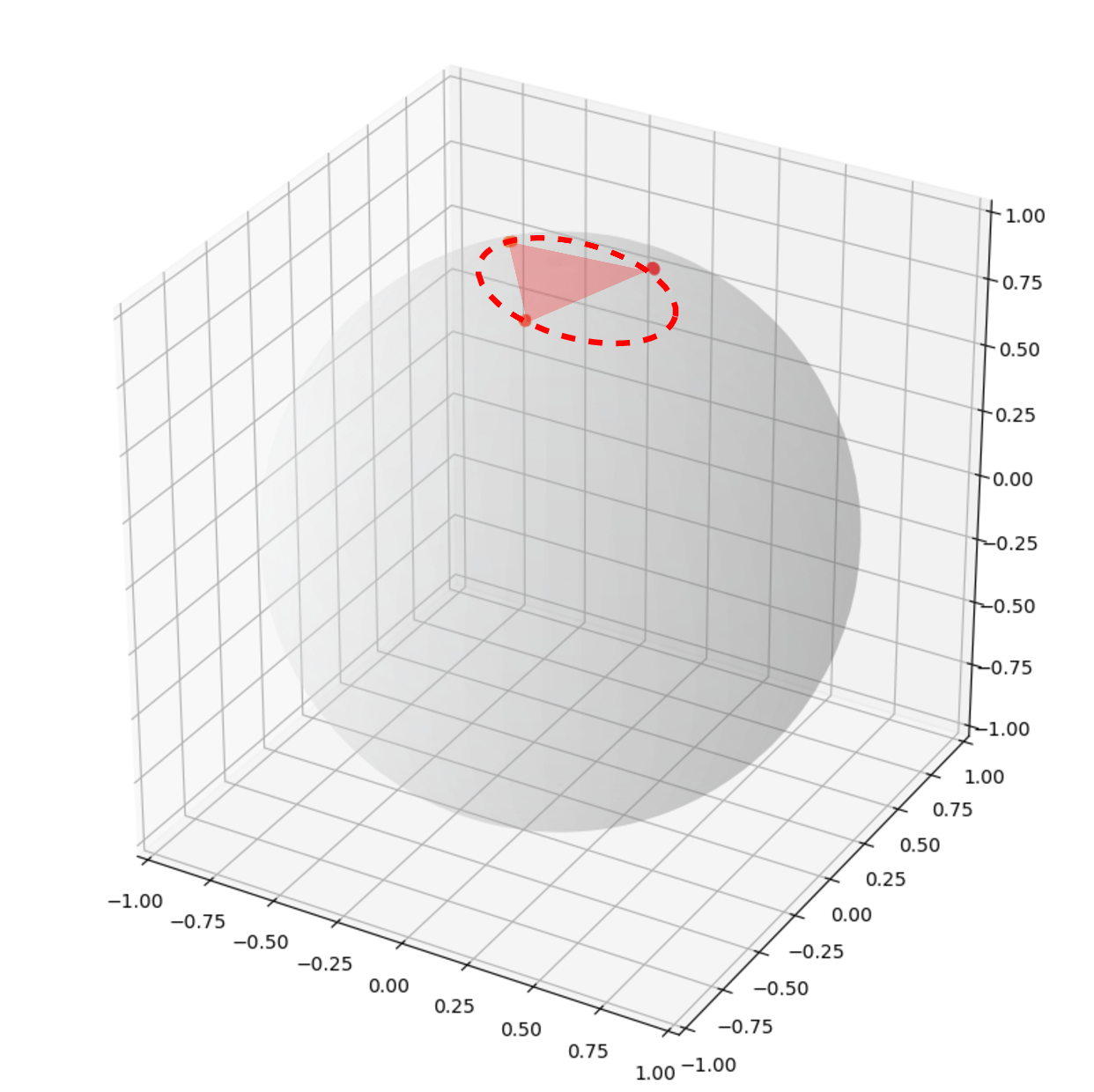}
    	\end{minipage}
	\label{fig:different_lamda_d}
    }
      \subfigure[$\mathcal{L}_{SupCon}$]{
    	\begin{minipage}[b]{0.13\textwidth}
   		\includegraphics[width=1\textwidth]{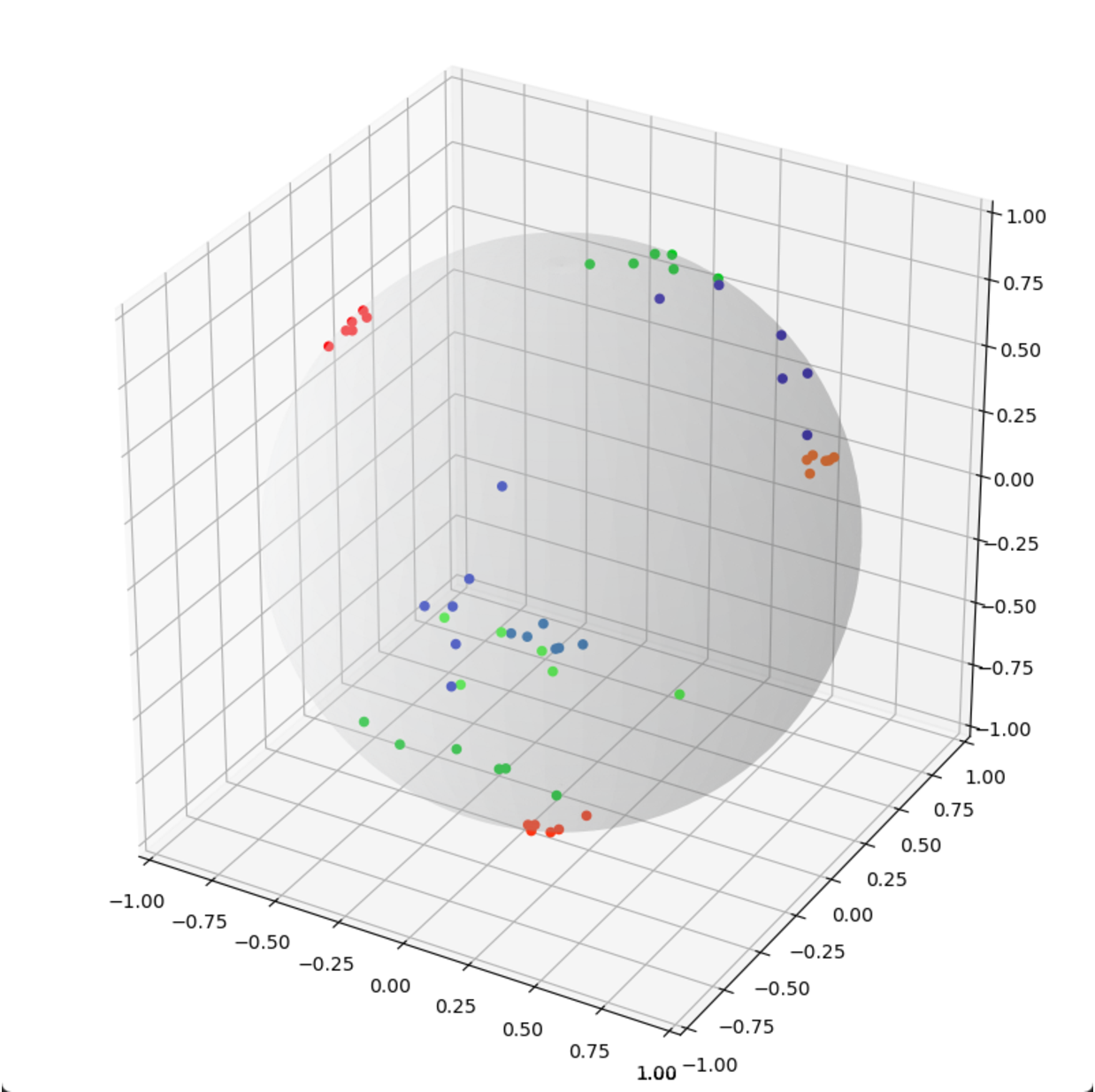}
    	\end{minipage}
	\label{fig:continual_3d_toy_example_e}
    }
    \subfigure[$\mathcal{L}_{R2SCL}$]{
    	\begin{minipage}[b]{0.13\textwidth}
   		\includegraphics[width=1\textwidth]{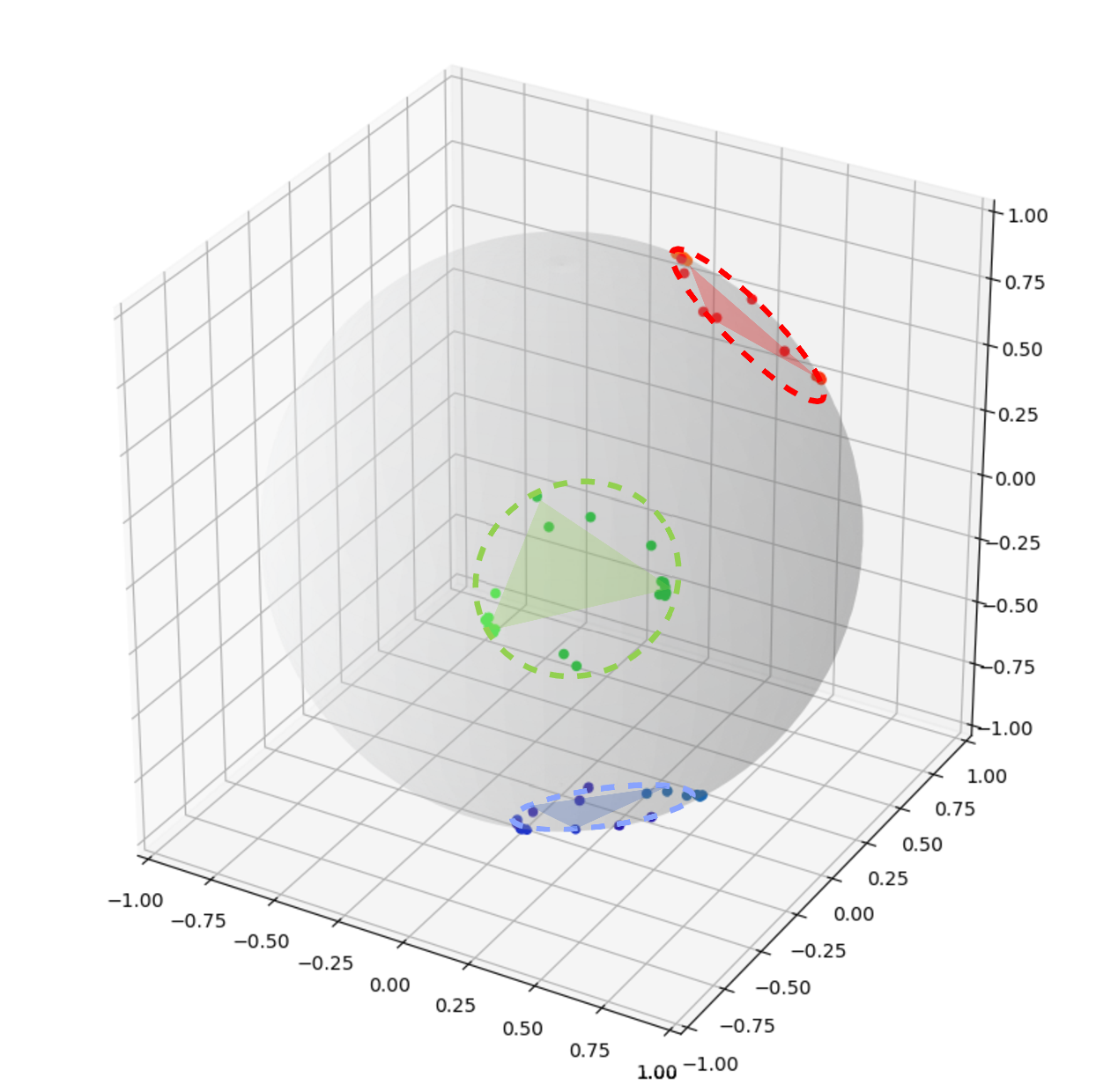}
    	\end{minipage}
	\label{fig:continual_3d_toy_example_f}
    }
	\caption{\textbf{A simple toy example that demonstrates the effectiveness of our method}.
    We use different colors to represent data from \textit{different tasks}. \textbf{(a,b,c,d):} Conventional SupCon versus ours in single task. The red triangles represent the final feature structure. \textbf{(e,f):} Different contrastive learning strategy in continual learning. 
    }
	\label{fig:continual_3d_toy_example}
\end{figure*}
\noindent\textbf{Global Pre-fixing:}
Merely intra-task discriminability is not sufficient to guarantee good performance in CIL. The inter-task performance is also significant. To some extent, the distance between task prototypes represents the inter-task separation. To maximize it, we design a task-level pre-allocated ETF, which is widely considered as the structure with the greatest separability\cite{fickus2018equiangular}.
We naturally assign the task prototypes to the vertices of the ETF, i.e. $P_{fix}^t=P_{ETF}^t$, where $P_{ETF}^t, t=1,...,T$ form a task-level ETF.

\noindent\textbf{Threshold Selection:} 

To better encourage more separation between adjacent task regions, we introduce $margin$ between the pre-allocated regions, as shown in \Cref{fig:threshold}.
This constrains the radius, $\sqrt{(1-\frac{1}{|\mcY|})(1-k)}$ , which is controlled by the lower bound of inter-task similarity $k$. We theoretically explain the condition that $k$ must satisfy when  $margin \ge 0$, i.e. $k \geq k_{min}$, where $k_{min}=1 - \frac{{|\mcY|}}{{|\mcY| - 1}}\sin {(\frac{1}{2}{\theta _{ETF}})^2}$. We provide a proof sketch in \Cref{fig:threshold}, with the complete proof included in the appendix. We can moderately increase the $margin$ to further improve separability.


\begin{figure}
	\centering
    \includegraphics[width=0.5\linewidth]{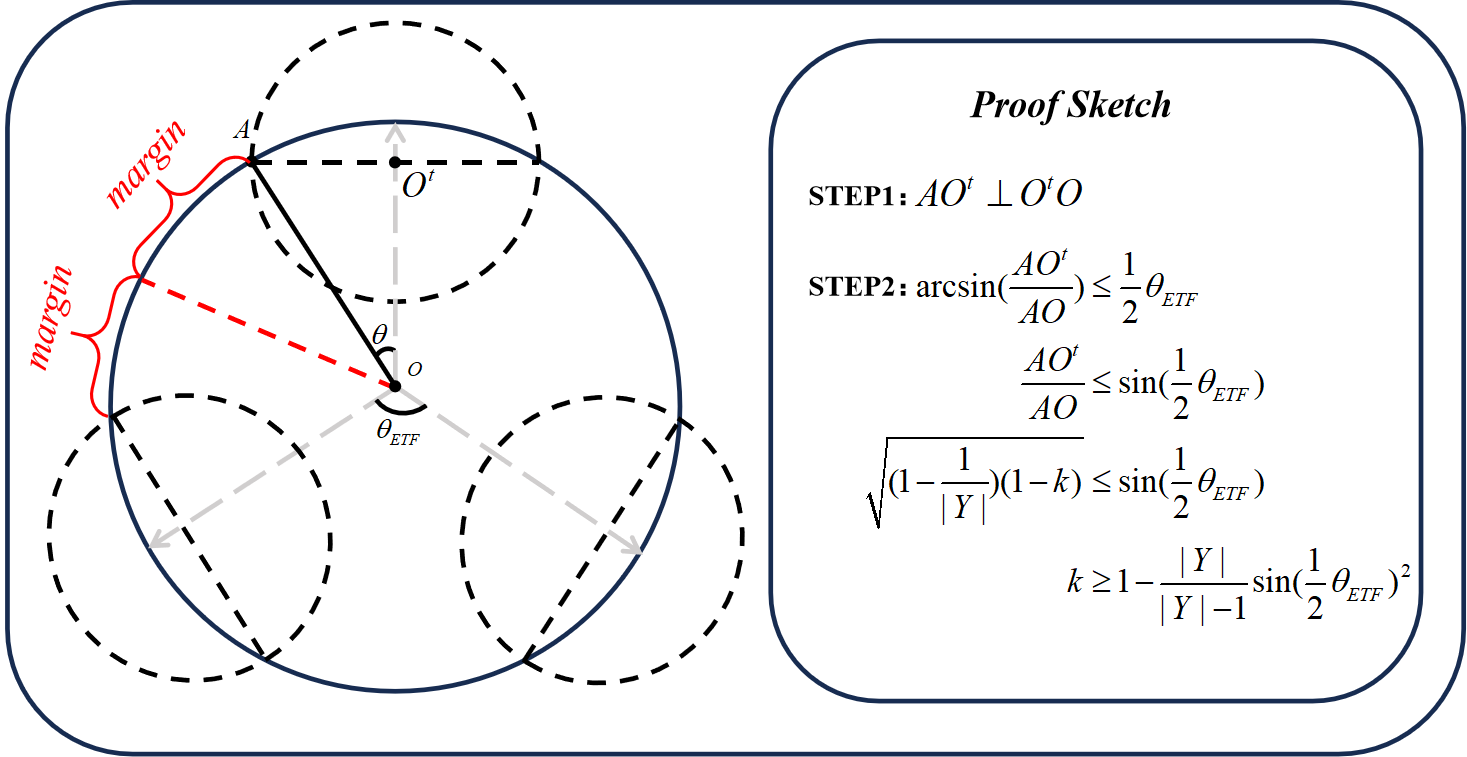}
	\caption{Illustration and proof sketch for threshold selection. $O^t$ represents the origin of the final inscribed hypersphere for task $t$. Because of the symmetry of the ETF, the threshold is globally shared.}
	\label{fig:threshold}
\end{figure}

\begin{equation}
   \begin{aligned}
    \label{eqn:thresholg_margin}
    k=(1-k_{min}) \cdot margin +k_{min}
    \nonumber
    \end{aligned} 
\end{equation}

where $ 0 \leq margin \leq 1$, representing additional gap introduced between regions.

\noindent\textbf{A Toy Example.}
We provide a toy example to show the effectiveness of our methods when there is no drift in the past features.

\noindent\textit{Toy Example Setting:} We follow the setting in \cite{graf2021dissecting}. Specifically, we used 3-dimensional data randomly sampled with Torch, which represents features $Z$ in deep model training. We apply the aforementioned \Cref{def:position_loss} to optimize $Z$ according to the gradients. In \Cref{fig:continual_3d_toy_example} (a, b, c, d), we present the final features for a single task with three classes, using different \textit{thresholds} and a fixed task prototype. In \Cref{fig:continual_3d_toy_example} (e, f), we first generate data for three tasks, each containing three classes. We use different colors to represent classes from different tasks and present the final features trained with SupCon (e) and ours (f) respectively.

\noindent\textit{Results:} 
As illustrated in \Cref{fig:continual_3d_toy_example}, we observe that SupCon suffers from features confusion, as it cannot spontaneously form a well-separated feature structure at the task level. In contrast, our methods have almost no overlap among tasks in our toy example. The reason is that our method can ensure both inter-task and intra-task well-separated feature structure.

\subsection{To Prevent Feature Drift: More Effective Feature-Level Distillation}
\label{sec:simple_distallation}

As tasks arrive continually, learned features will gradually drift away from their pre-fixed regions\cite{sdc}. Existing  contrastive  continual learning methods\cite{cha2021co2l,wen2024provable} usually adopt an instance relation distillation to prevent relative drift between features. To better prevent such feature drift in past tasks, we empirically explore which types of distillation methods are most effective in contrastive continual learning. 
We evaluate several commonly used distillation techniques\cite{bhat2023task,sarfraz2023error} in CL as well as their combinations, with results presented in \Cref{fig:distall_ablation}. Our findings show that incorporating MSE loss into existing contrastive continual learning frameworks yields the best performance. To minimize interference with the current task, we apply distillation using samples from the memory buffer rather than from the current training batch.
\begin{equation}
\scriptsize 
   \begin{aligned}
    \label{equ:mse_distall}
    \mathcal{L}_{distill}(D_M^{t-1};f_e^t,f_e^{t-1})\triangleq \lambda_{distill} \sum \nolimits_{x_i \in D_M^{t-1}} || f_e^t(x_i)-f_e^{t-1}(x_i)||^2
    \nonumber
    \end{aligned} 
\end{equation}

\noindent\textbf{Total Loss:}
\begin{equation}
   \begin{aligned}
\label{eqn:total}\mathcal{L}_{GPLASC}=\mathcal{L}_{range}+\mathcal{L}_{position}+\mathcal{L}_{distill}
\nonumber
    \end{aligned} 
\end{equation}

It addresses the challenges of CL from three aspects: $\mathcal{L}_{range}$ tackles \textit{intra-task feature confusion} by contrastive learning in the task-specific region, $\mathcal{L}_{position}$  tackles \textit{inter-task feature confusion} by pre-allocating regions and maximizing separability between region centers, and $\mathcal{L}_{distill}$  tackles \textit{forgetting} from feature level by introducing MSE loss.

%

\section{Experimental Setup}
\label{sec:experiments}
We follow most of the continual learning settings, model architecture, data preparation, and other training details from the past literature\cite{cha2021co2l,cclis,tiwari2022gcr}.

\noindent\textbf{Datasets}. 
We conduct experiments on Seq-CIFAR-10, Seq-CIFAR-100, and Seq-Tiny-ImageNet under both Task-Incremental Learning (TIL) and Class-Incremental Learning (CIL) settings. These datasets are widely adopted benchmarks in continual learning due to their increasing complexity.

\noindent\textbf{Seq-CIFAR-10} is derived from the CIFAR-10 dataset, which contains 60,000 $32 \times 32$ color images  distributed across 10 classes. Following the protocol in\cite{cclis}, we split the dataset into 5 disjoint tasks, each containing 2 classes. This dataset is relatively simple and is often used to validate the basic functionality of continual learning methods.

\noindent\textbf{Seq-CIFAR-100} is based on CIFAR-100 \cite{cifar}, which has 60,000 $32 \times 32$ images across 100 classes. It is divided into 5 tasks, with 20 classes per task. Compared to CIFAR-10, it introduces higher semantic overlap between classes, making it more challenging for feature discrimination and knowledge retention.

\noindent\textbf{Seq-Tiny-ImageNetis} built from the Tiny-ImageNet dataset, consisting of 100,000 images of resolution $64 \times 64$, spanning 200 object categories. Following prior work, we divide it into 10 sequential tasks, each containing 20 disjoint classes. This dataset poses greater challenges due to its higher visual diversity and larger class count.

All datasets are split in a class-incremental fashion and used consistently for both TIL and CIL experiments. The task sequences are fixed across all compared methods to ensure fair evaluation.

\noindent\textbf{Settings}. For the backbone, we use ResNet-18, which is a common choice in many previous methods\cite{cha2021co2l,cclis,tiwari2022gcr}. We adopt a two-stage training strategy, similar to \cite{cha2021co2l}: first, we train on the current task's data, then we freeze the feature extraction part of the model and train the classifier on both the current task and exemplar set. To avoid suffering from the class-imbalanced issue as much as possible, we follow the training approach in \cite{cha2021co2l}. In the following, we first assume that the specific number of tasks is also known, and in the ablation study, we explore the impact on performance when the task count is unknown and a sufficiently large task number is chosen instead.

\noindent\textbf{Evaluation}. For evaluation metrics, we use both \textit{Accuracy}  and \textit{Average Forgetting}. Details on hyperparameters tuning can be found in our appendix. Due to space limitations, additional experiments about \textit{Average Forgetting} and \textit{Accuracy} evaluation with memory sizes 500 is provided in the appendix.

%
\noindent\textbf{Baselines}. We select several classic baselines for comparison, ER\cite{er}, GEM \cite{agem}, A-GEM \cite{agem},iCaRL \cite{icarl}, FDR \cite{fdr} ,  DER \cite{der},DER++ \cite{der}. The choice of these baselines is consistent with prior studies\cite{cha2021co2l}. We also compared our approach with several recent works\cite{cha2021co2l,cclis,wen2024provable} that specifically focus on contrastive continual learning. 
All baseline results, except for \cite{wen2024provable}, are collected from \cite{cclis}.
Note that \cite{wen2024provable} do not have open source code. We try to reproduce the implementation ourselves.

\noindent\textbf{Results and Discussion}. We explore whether our method can learn high-quality embeddings to overcome feature confusion. In \Cref{tab:acc_result}, all existing contrastive continual learning methods show improved accuracy in CIL after integrating our approach.
Specifically, with 200 examples, we improved the CIL accuracy of Co2l by $\approx 5\% $ on CIFAR-10 and by $\approx 5\% $ on CIFAR-100. This improvement is significant because methods like Co2l achieve discriminative features between tasks only with insufficient negative samples, leading to poor inter-task separability. Our method simultaneously ensures discriminative feature structure both inter-task and intra-task. In addition to improving CIL performance, in most cases, our method also improves TIL performance. The reason is that although the intra-task feature region is constrained, feature can still converge to a local regular simplex within the allocated region. Meanwhile, the feature-level MSE loss prevent feature drift effectively as tasks sequentially coming.
\begingroup

A deeper analysis reveals that the efficacy of GPLASC hinges on the dataset's inherent characteristics, particularly the balance between the need for inter-task separation and the challenge of intra-task discrimination:

\begin{itemize}
\item When GPLASC Excels (e.g., CIFAR-10/100): GPLASC is most effective when the primary challenge is inter-task confusion (i.e., catastrophic forgetting). In scenarios with fewer classes per task (like CIFAR-10/100), our "Global Pre-fixing" strategy provides a strong, stable structure that clearly separates task representations. This structural prior is highly beneficial, as there is ample feature space for the "Local Adjusting" mechanism to organize the classes within each task.
\item When its Advantage is Less Pronounced (e.g., Tiny ImageNet): The advantage of GPLASC diminishes when intra-task complexity becomes the dominant bottleneck. On datasets like Tiny ImageNet with many difficult classes per task, the pre-allocated feature sub-region for each task becomes crowded. The rigidity of our global structure, which is an asset in simpler settings, becomes a constraint. It limits the model's flexibility to learn optimal representations for a large number of visually similar classes, a phenomenon we term "intra-task feature crowding."

\end{itemize}


\subsection{Generality and Compatibility with Mainstream Replay Methods}
\label{sub:replay_based}
To further investigate the generality and compatibility of our proposed GPLASC, we conducted a series of experiments to evaluate whether it can serve as a 'plug-and-play' module to enhance existing mainstream replay-based methods. This directly addresses the question of whether the benefits of GPLASC are complementary to established continual learning strategies.

\noindent\textbf{Experimental Setup}: We integrated GPLASC into three highly influential and competitive replay-based baselines:
\begin{itemize}
\item \textbf{Experience Replay (ER)}: The foundational replay algorithm.
\item \textbf{iCaRL}: A classic method that combines replay with a nearest-mean-of-exemplars classifier.
\item \textbf{DER}: A strong and simple continual learning baseline that mitigates forgetting by replaying both past samples and their corresponding model outputs (logits).

\end{itemize}

For each baseline, we augmented its original loss function with our GPLASC-regulated contrastive loss. This auxiliary loss was computed on the feature embeddings extracted from a combined batch, which included only samples from the current task. All other experimental settings, such as memory size, model architecture, and optimization hyperparameters, were kept identical to our main experiments for a fair comparison. We evaluated these augmented methods on the CIFAR-10, CIFAR-100, and Tiny-ImageNet benchmarks.

As shown in the \Cref{tab:er_baseline_}, augmenting the standard ER with GPLASC yields a notable accuracy gain of 2.7\% on CIFAR-100. More importantly, even a very strong baseline like DER benefits from the integration of GPLASC, achieving a performance boost of 1.9\% on CIFAR-100 and 1.9\% on Tiny-ImageNet. This synergistic effect highlights that our method's mechanism—enforcing a more structured and discriminative feature space—is complementary to existing replay strategies. While replay methods preserve knowledge by storing past samples, GPLASC ensures that the representations of both new and replayed samples are organized in a way that minimizes inter-task interference.

\begin{table*}[t]
\centering
\small 
\setlength{\tabcolsep}{3pt} 
\begin{tabularx}{\textwidth}{lXXXXXX} 
\toprule
\textbf{Method} & \textbf{CIFAR-10 (CIL)} & \textbf{CIFAR-10 (TIL)} & 
\textbf{CIFAR-100 (CIL)} & \textbf{CIFAR-100 (TIL)} & 
\textbf{Tiny-ImgNet (CIL)} & \textbf{Tiny-ImgNet (TIL)} \\
\midrule
ER              & 58.21±0.81 & 89.34±0.45 & 26.29±1.22 & 60.97±0.94 & 11.10±0.25 & 37.26±0.83 \\
\textbf{ER + Ours} & \textbf{61.51±0.75} & \textbf{90.64±0.41} & \textbf{29.79±1.10} & \textbf{63.47±0.88} & \textbf{12.80±0.35} & \textbf{40.06±0.79} \\
\midrule
iCaRL           & 32.44±0.93 & 74.59±1.24 & 28.00±0.91 & 51.43±1.47 & 5.50±0.52  & 22.89±1.83 \\
\textbf{iCaRL + Ours}& \textbf{35.94±1.02} & \textbf{75.39±1.21} & \textbf{32.80±1.15} & \textbf{53.73±1.55} & \textbf{5.65±0.78}  & \textbf{25.89±1.70} \\
\midrule
DER             & 63.69±2.35 & \textbf{91.91±0.51} & 31.23±1.38 & 63.09±1.09 & 13.22±0.92 & 42.27±0.90 \\
\textbf{DER + Ours} & \textbf{63.99±2.55} & 91.81±0.65 & \textbf{33.03±1.25} & \textbf{63.59±1.18} & \textbf{14.12±0.99} & \textbf{42.57±1.05} \\
\bottomrule
\end{tabularx}
\caption{Effectiveness of Our Method on Replay-Based Continual Learning Baselines}
\label{tab:er_baseline_}
\end{table*}

\begin{table*}[!t]
    \centering
    \small
    \begin{tabular}{cccccccc}
	\hline

\multirow{2}{*}
{\textbf{Buffer}}&\textbf{Dataset}&\multicolumn{2}{c}{\textbf{Seq-Cifar-10}}&\multicolumn{2}{c}{\textbf{Seq-Cifar-100}}&\multicolumn{2}{c}{\textbf{Seq-Tiny-ImageNet}}\\
            &\textbf{Scenario}&\textbf{Class-IL}&\textbf{Task-IL}&\textbf{Class-IL}&\textbf{Task-IL}&\textbf{Class-IL}&\textbf{Task-IL}\\
\hline
\multirow{10}{*}{200}&ER&49.16$\pm$2.08&91.92$\pm$1.01&21.78$\pm$0.48&60.19$\pm$1.01&8.65$\pm$0.16&38.83$\pm$1.15\\  
&iCaRL&32.44$\pm$0.93&74.59$\pm$1.24&28.0$\pm$0.91&51.43$\pm$1.47&5.5$\pm$0.52&22.89$\pm$1.83\\
    &GEM&29.99$\pm$3.92&88.67$\pm$1.76&20.75$\pm$0.66&58.84$\pm$1.00&-&-\\
    &GSS&38.62$\pm$3.59&90.0$\pm$1.58&19.42$\pm$0.29&55.38$\pm$1.34&8.57$\pm$0.13&31.77$\pm$1.34\\
    &DER&63.69$\pm$2.35&91.91$\pm$0.51&31.23$\pm$1.38&63.09$\pm$1.09&13.22$\pm$0.92&42.27$\pm$0.90\\
    &GCR&64.84$\pm$1.63&90.8$\pm$1.05&33.69$\pm$1.40&64.24$\pm$0.83&13.05$\pm$0.91&42.11$\pm$1.01\\
    
    \cmidrule[0.5pt]{2-8}
    &[ICCV, 2021] Co2L&65.57$\pm$1.37&93.43$\pm$0.78&27.73$\pm$0.54&54.33$\pm$0.36&13.88$\pm$0.40&42.37$\pm$0.74\\
    
    &Co2L + Ours&\textbf{70.59$\pm$1.26}&\textbf{95.55$\pm$0.37 }&\textbf{32.48$\pm$1.08}&\textbf{62.01$\pm$1.24}&\textbf{14.30$\pm$0.51}&\textbf{44.53$\pm$0.83}\\
    \cmidrule[0.5pt]{2-8}
    &[ICML, 2024] CILA&67.06$\pm$1.59&94.29$\pm$0.24&30.18$\pm$0.39&58.19$\pm$0.28&14.55$\pm$0.39&\textbf{44.15$\pm$0.70}\\

    &CILA + Ours&\textbf{71.8$\pm$1.35}&\textbf{96.41$\pm$0.86}&\textbf{33.13$\pm$1.34}&\textbf{62.13$\pm$1.88}&\textbf{15.51$\pm$0.47}&44.12$\pm$0.65\\

    \cmidrule[0.5pt]{2-8}
    &[AAAI, 2024] CCLIS&74.95$\pm$0.61&96.20$\pm$0.26&42.39$\pm$0.37&\textbf{72.93$\pm$0.46}&16.13$\pm$0.19&48.29$\pm$0.78\\

    &CCLIS + Ours&\textbf{76.33$\pm$1.14}&\textbf{96.73
$\pm$0.48}&\textbf{44.48$\pm$1.27}&72.91$\pm$1.45&\textbf{17.16$\pm$0.34}&\textbf{48.80$\pm$0.68} \\
\hline
\multirow{10}{*}{500}&ER&62.03$\pm$1.70&93.82$\pm$0.41&27.66$\pm$0.61&66.23$\pm$1.52&10.05$\pm$0.28&47.86$\pm$0.87\\
    &iCaRL&34.95$\pm$1.23&75.63$\pm$1.42&33.25$\pm$1.25&58.16$\pm$1.76&11.0$\pm$0.55&35.86$\pm$1.07\\
    &GEM&29.45$\pm$5.64&92.33$\pm$0.80&25.54$\pm$0.65&66.31$\pm$0.86&-&-\\
    &DER&72.15$\pm$1.31&93.96$\pm$0.37&41.36$\pm$1.76&71.73$\pm$0.74&19.05$\pm$1.32&53.32$\pm$0.92\\

    &GCR&74.69$\pm$0.80&94.44$\pm$0.32&45.91$\pm$1.30&71.64$\pm$2.10&19.66$\pm$0.68&52.99$\pm$0.89\\

\cmidrule[0.5pt]{2-8}
    &Co2L&74.26$\pm$0.77&95.90$\pm$0.26&36.39$\pm$0.31&63.97$\pm$0.42&20.12$\pm$0.42&53.04$\pm$0.69\\

    &Co2L + Ours&\textbf{77.24$\pm$0.54}&\textbf{96.8$\pm$0.38}&\textbf{40.09$\pm$1.23}&\textbf{67.21$\pm$0.25}&\textbf{22.34$\pm$0.34}&\textbf{53.30$\pm$0.86}\\

\cmidrule[0.5pt]{2-8}

    &CILA&77.03$\pm$0.79&96.40$\pm$0.21&37.06$\pm$0.84&66.48$\pm$0.26&20.64$\pm$0.59&54.13$\pm$0.72\\

    &CILA + Ours&\textbf{79.03$\pm$1.12}&\textbf{96.12$\pm$0.66}&\textbf{40.76$\pm$1.28}&\textbf{63.8$\pm$0.52}&\textbf{21.35$\pm$0.47}&\textbf{54.13$\pm$0.35}\\

\cmidrule[0.5pt]{2-8}
&CCLIS&78.57$\pm$0.25&96.18$\pm$0.43&46.08$\pm$0.67&\textbf{74.51$\pm$0.38}&22.88$\pm$0.40&57.04$\pm$0.43\\
    &CCLIS + Ours&\textbf{79.08$\pm$1.31}&\textbf{96.30$\pm$0.11}&\textbf{47.38$\pm$1.31}&73.35$\pm$0.89&\textbf{23.62$\pm$0.64}&\textbf{57.12$\pm$0.59}\\    
\hline
\end{tabular}

    \caption{Class-IL and Task-IL Continual Learning. We report our performance and the results of rehearsal-based baselines on Seq-Cifar-10, Seq-Cifar-100 and Seq-Tiny-ImageNet with memory sizes 200 and 500, all of which are averaged across five independent trails.
    }
\label{tab:acc_result}
\end{table*}

\subsection{Evaluation on Challenging Benchmarks to Verify Generalizability and Robustness}
\label{sub:eval_cub}
\begin{table*}
    \centering
    \small
    \begin{tabular}{cccccccc}
    \hline
    \multirow{2}{*}
    {\textbf{Buffer}}&\textbf{Dataset}&\multicolumn{2}{c}{\textbf{CUB200}}&\multicolumn{2}{c}{\textbf{Imagenet-R}}&\multicolumn{2}{c}{\textbf{Imagenet-A}}\\
    &\textbf{Scenario}&\textbf{Class-IL}&\textbf{Task-IL}&\textbf{Class-IL}&\textbf{Task-IL}&\textbf{Class-IL}&\textbf{Task-IL}\\
    \hline
    \multirow{2}{*}{200}
    &Co2l
    &8.21±1.37 &19.35±0.63 &9.33±0.79 &12.67±1.93 &\textbf{7.8±1.13} &9.4±0.75  
    \\ 
    
    &Co2l + ours
    &\textbf{10.20±0.92} &\textbf{20.81±1.74} &\textbf{11.20±1.99} &\textbf{14.27±1.99} &7.2±1.59 &\textbf{10.3±1.12}
    \\ 
    \hline
    \multirow{2}{*}{500}
      &Co2l
    &14.52±1.85 &25.71±1.51 &12.34±1.66 &20.32±0.73 &\textbf{11.2±1.97} &\textbf{13.5±2.12}
    \\ 
    
    &Co2l + ours
    &\textbf{15.32±1.20}	&\textbf{27.32±1.12}	&\textbf{14.94±1.13}	&\textbf{22.21±0.98}	&10.5±2.83	&12.9±1.21
    \\ 
    \hline
    \end{tabular}
    \caption{Class-IL and Task-IL Continual Learning. We report our performance and the results of rehearsal-based baselines on CUB200, Imagenet-R and ImageNet-A with memory sizes 200 and 500, all of which are averaged across five independent trails}
    \label{tab:cub200_imagenet_r_a}
\end{table*}

\textbf{Datasets}. We selected the following three datasets to rigorously test specific capabilities of our model:

\begin{itemize}
    \item CUB-200-2011 (CUB-200) \cite{cub200}: A challenging fine-grained visual classification dataset containing 11,788 images of 200 bird species. This benchmark evaluates the model's ability to learn subtle and nuanced distinctions between visually similar classes, a critical aspect for real-world applications. We divide the dataset into five tasks, each containing 40 classes.
    \item ImageNet-R (Renditions) \cite{img_r}: A dataset comprising 30,000 images of 200 ImageNet classes with various artistic renditions (e.g., cartoons, graffiti, paintings). We divide the dataset into five tasks, each containing 40 classes.
    \item ImageNet-A (Adversarial) \cite{img_a}: A collection of real-world, unmodified, and naturally adversarial examples that cause significant performance degradation in state-of-the-art classifiers. Success on this dataset indicates strong model robustness against challenging, real-world perturbations. We divide the dataset into five tasks, each containing 40 classes.
\end{itemize}

\noindent\textbf{Implementation Details}. A crucial aspect of our evaluation is the training protocol. To rigorously assess the intrinsic learning capability of our method, all models were trained entirely from scratch, without relying on any pre-trained backbones (e.g., pre-trained on ImageNet). This represents a significantly more challenging and realistic continual learning setting, as the model must build meaningful representations from the ground up.
Due to the substantial computational and memory demands of training from scratch on these large-scale datasets, particularly when maintaining a rehearsal buffer (we used a buffer of 200 samples), we adopted an image resolution of 224*224 for all experiments in this section. To ensure a fair and direct comparison, we subjected our strong baseline, Co2L \cite{cha2021co2l}, to the exact same demanding conditions, including the from-scratch training protocol and the 224*224 input resolution. We used a ResNet-18 as the backbone for all runs.

\noindent\textbf{Results}. As shown in \Cref{tab:cub200_imagenet_r_a}, our method outperforms existing approaches across most benchmarks.
\begin{itemize}
\item On CUB-200, the superior performance of our method underscores its advanced capability to learn and preserve fine-grained features in a continual setting, which is a known difficulty for existing methods.
\item On ImageNet-R, our method achieves significantly better performance compared to existing approaches, highlighting its strong robustness and generalization under domain shifts. 
\item However, the performance drops notably on the more challenging ImageNet-A dataset, suggesting that our approach still struggles with natural adversarial examples.We identify two potential causes for the performance degradation of our method on ImageNet-A: 1.(Data Scarcity and Overfitting) Each task contains an insufficient number of training samples, which elevates the risk of the model overfitting to the limited data. 2.(Inherent Task Difficulty) The task of learning discriminative features within ImageNet-A is a significant challenge. We argue that this intra-task difficulty is the primary limiting factor for its continual learning performance.
\end{itemize}

\endgroup

\section{Ablation}

\subsection{Effectiveness of Components}
To validate the effectiveness of our methods, we introduce three variants of GPLASC: Without $\mathcal{L}_{R2SCL}$ 
 , where we train only with baseline Co2l and $\mathcal{L}_{distill}$; with $\mathcal{L}_{R2SCL}$ only, where we optimize the model without $\mathcal{L}_{distill}$; baseline Co2l. We compare our methods with the three variants in Seq-Cifar-10 with 200 buffered samples and show the results in \Cref{tab:ablation_components}. 
\begin{table}
    \centering
    \small
    \caption{Ablation study of R2SCL and feature-level distillation. We train our model on the Seq-CIFAR-10 dataset with 200 buffered samples under CIL and TIL scenario to explore the effectiveness of components.}
    \label{tab:ablation_components}
    \begin{tabular}{lcccc}
        \toprule
        & \textbf{Distillation} & \textbf{R2SCL} & \textbf{Acc (CIL)} & \textbf{Acc (TIL)} \\
        \midrule
        Co2l                        & \ding{55} & \ding{55} & 65.57 & 93.43 \\
        w/ $\mathcal{L}_{distill}$ only 
                                    & \ding{51} & \ding{55} & 67.86 & 94.12 \\
        w/ $\mathcal{L}_{R2SCL}$ only 
                                    & \ding{55} & \ding{51} & 67.31 & 94.68 \\
        w/ $\mathcal{L}_{R2SCL}$+$\mathcal{L}_{distill}$ 
                                    & \ding{51} & \ding{51} & \textbf{70.59} & \textbf{95.55} \\
        \bottomrule
    \end{tabular}
\end{table}

Experimental results show that only incorporating feature-level distillation or R2SCL can improve baseline performance. However, combining both components yields the greatest performance enhancement. 
The reason is that our method adopts a \textit{global pre-fixing, local adjusting} strategy, which \textit{simultaneously} ensures discriminative feature structure both among tasks and within tasks.
Meanwhile, by introducing an additional feature-level MSE loss, the model can better prevent feature drift as tasks arrive continually.



\begin{figure}[!h]
    \centering
    \includegraphics[width=0.5\linewidth]{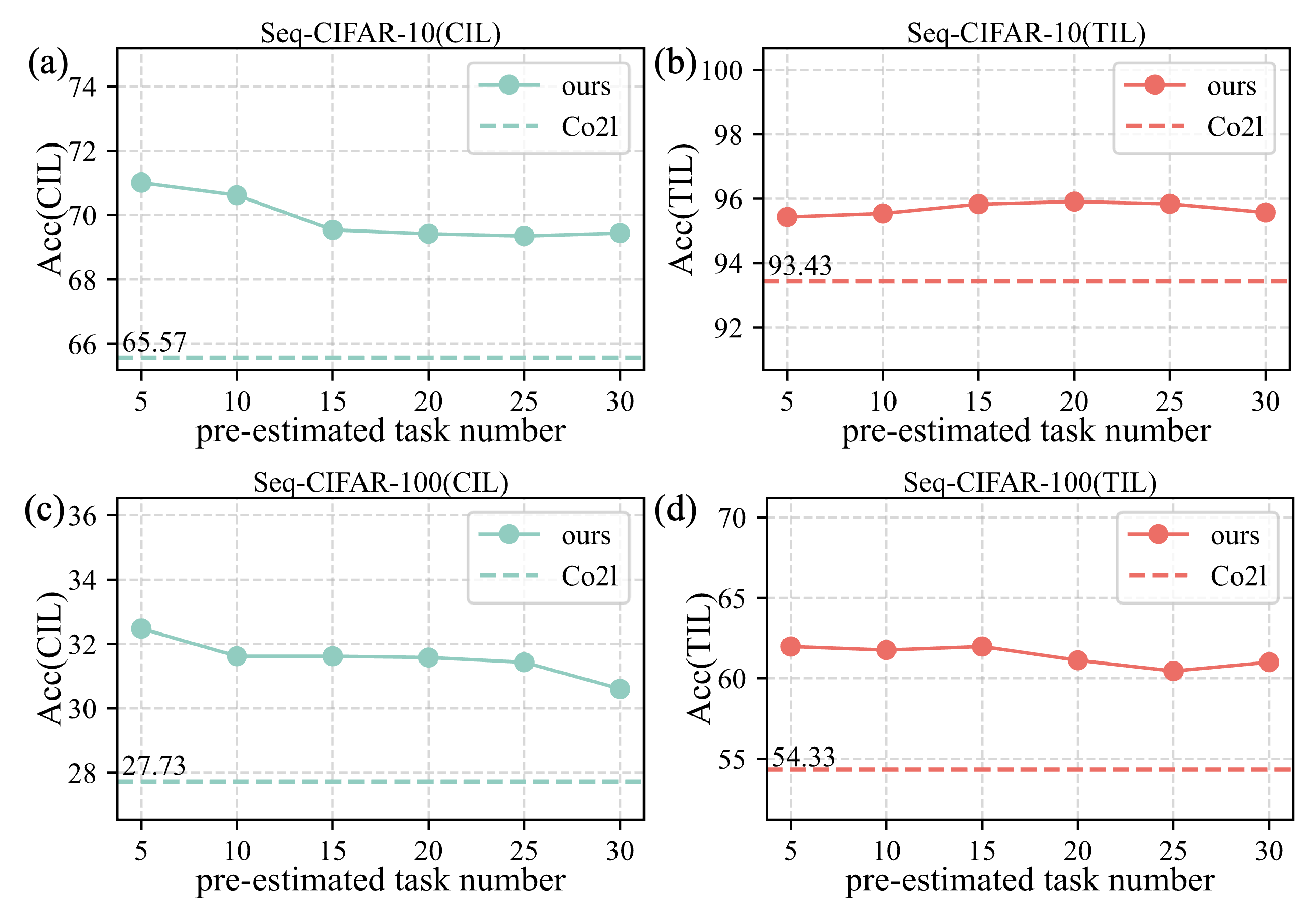}
    \caption{Performance with the different pre-estimated task numbers
    under CIL and TIL scenarios in Seq-CIFAR-10 and Seq-CIFAR-100.
    (a) CIL performance in Seq-CIFAR-10. 
    (b) TIL performance in Seq-CIFAR-10.
    (c) CIL performance in Seq-CIFAR-100.
    (d) TIL performance in Seq-CIFAR-100.
    }
    \label{fig:task_number_sensitivity}
\end{figure}

\subsection{Effectiveness of various Distillation Methods}
\label{sec:distillation_figs}
    On CIFAR-10, we select the standard SupCon as our baseline with 200 exemplars and apply our proposed R2SCL, as well as different distillation techniques:
    \textit{Projector}( Following \cite{shiprospective}, we use a simple under-complete autoencoder as projector. It consists of a linear layer followed by ReLU activation that maps the features to a low-dimensional subspace and another linear layer followed by sigmoid activation that maps the features back to high dimensions), \textit{IRD}(The distillation method used in \cite{graf2021dissecting} mainly involves distilling the relationships between positive and negative sample pairs),
    \textit{MSE}(For the same sample, the Mean Squared Error (MSE) is calculated between the features from the current and past encoders),
    \textit{IRD+MSE},
    \textit{IRD+Projector}.

\begin{figure}[H]
    \centering
    \includegraphics[width=0.5\linewidth]{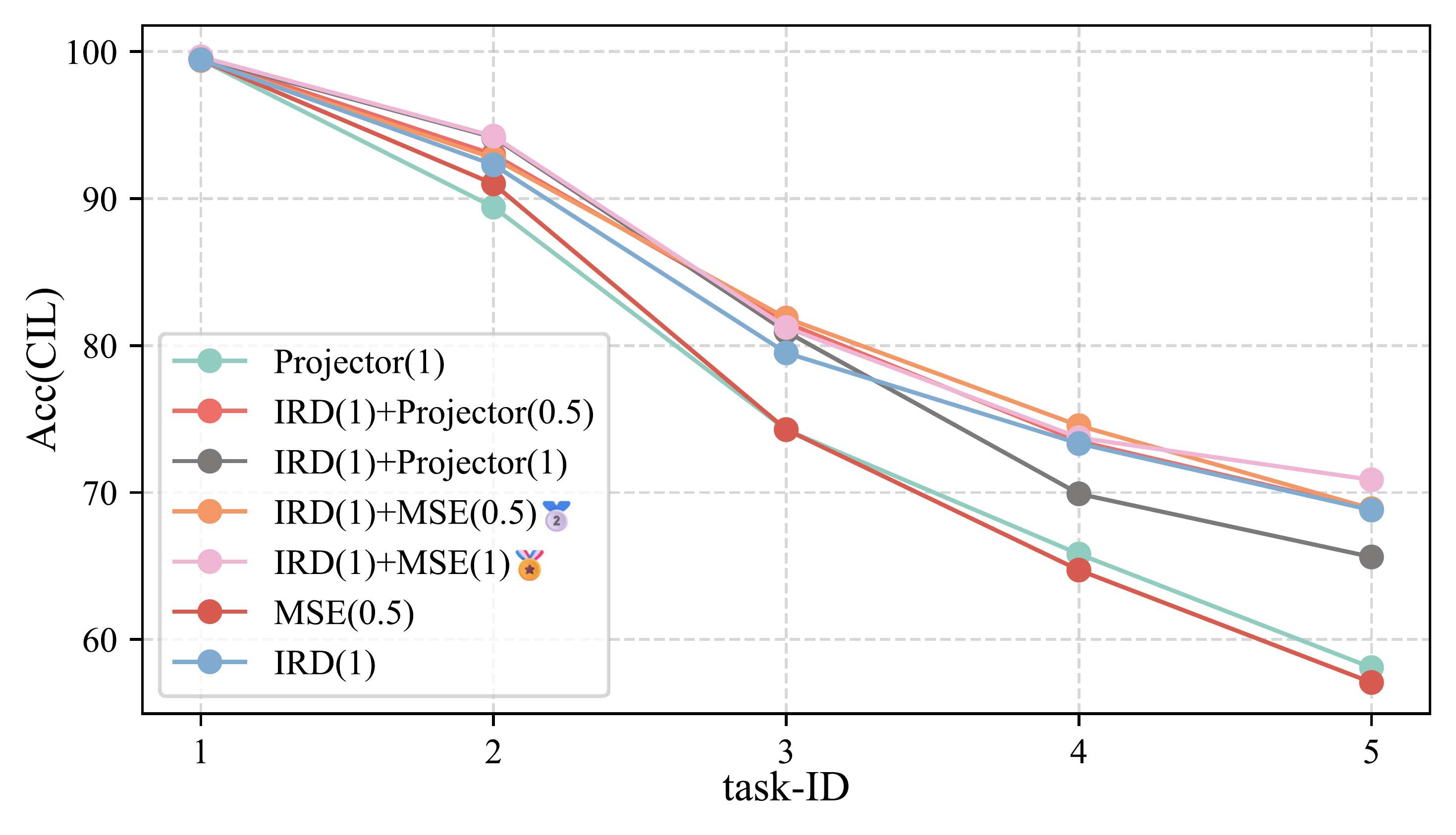}
    \caption{Performance with the different distillation methods and distillation power shown in $(\cdot)$ under CIL scenario in Seq-CIFAR-10. IRD(1)+MSE(1) stands out.}
    \label{fig:distall_ablation}
\end{figure}

Existing contrastive continual learning methods typically use \textit{IRD} as the distillation approach. As shown in \Cref{fig:distall_ablation}, \textit{IRD} demonstrates excellent performance to prevent feature drift. The reason is that \textit{IRD} distills the relative relationships between features, ensuring that the drift in relative positions is minimized. The combination of \textit{IRD} and \textit{MSE} achieved even better performance. The reason is that feature-level MSE distillation helps reduce the absolute drift of features. Our method, utilizing both types of distillation, minimizes feature drifting from their pre-fixed regions. 

\subsection{Effectiveness of Pre-Estimated Task Number}

We conduct experiments on Seq-CIFAR-10 and Seq-CIFAR-100 with Co2l to explore the impact of different pre-estimated task numbers. 
\textbf{Intuitively, a large number of pre-estimated tasks may compress the optimization space of features within tasks (similarly for CIL). However, a compressed feature space does not necessarily lead to a decline in discriminability — that is, it is still possible to form clear decision boundaries between classes within a task.}
Our results shown in \Cref{fig:task_number_sensitivity} indicate that larger pre-estimated task numbers do not significantly affect CIL and TIL performance. 
For TIL, performance remains almost unchanged. 
The underlying reason is that although our method restricts regions for individual tasks, it can still form a discriminative feature structure.


\begin{figure}[H]
    \centering
    \includegraphics[width=0.5\linewidth]{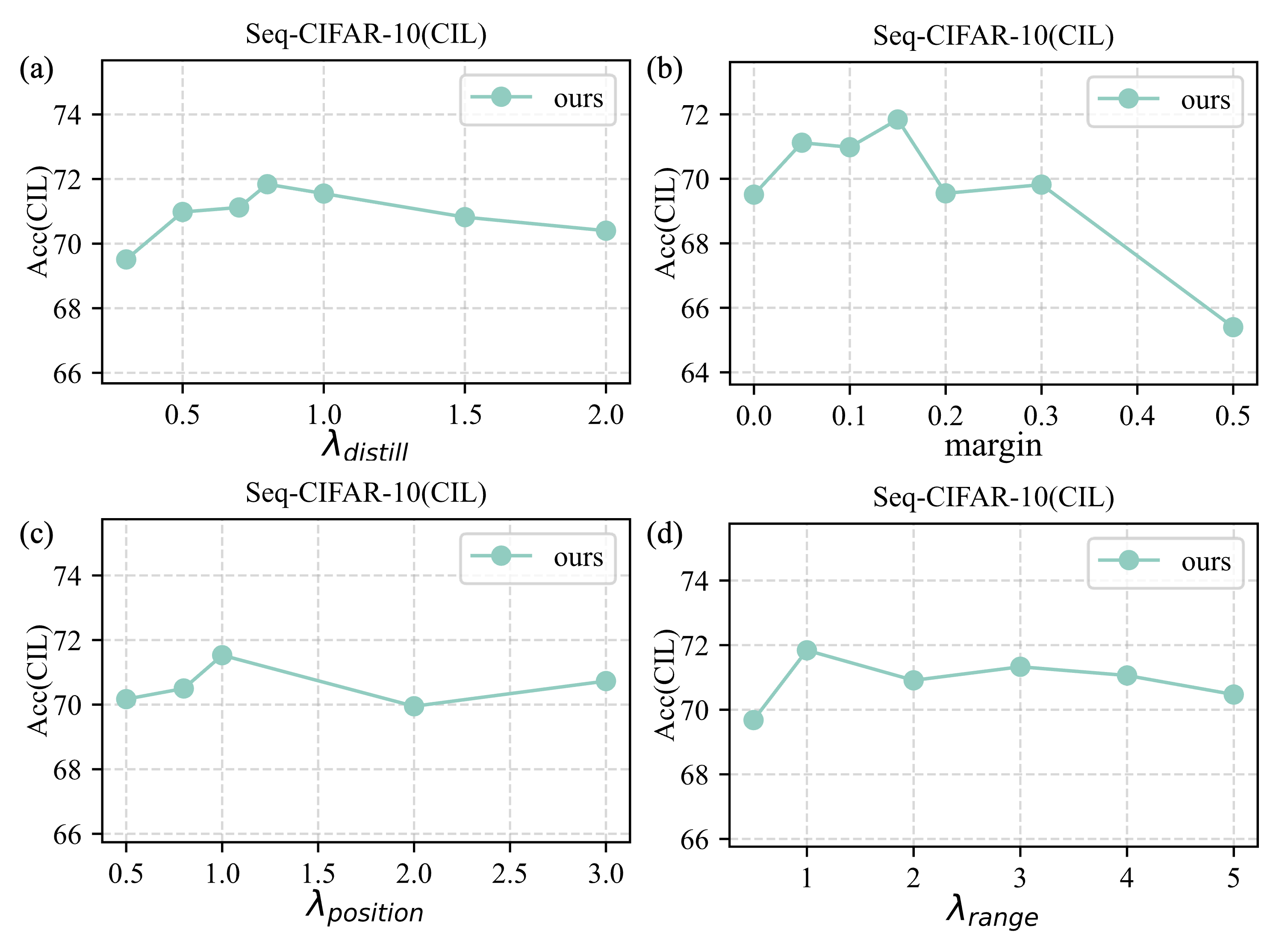}
    \caption{
        Sensitivity analysis is performed on four hyperparameters: 
        $\lambda_{position}$, $\lambda_{distill}$, $\lambda_{range}$, \textit{margin}. 
        (a) CIL performance on CIFAR-10 w.r.t. $\lambda_{position}$.
        (b) CIL performance on CIFAR-10 w.r.t. $\lambda_{distill}$.
        (c) CIL performance on CIFAR-10 w.r.t. $\lambda_{range}$.
        (d) CIL performance on CIFAR-10 w.r.t. \textit{margin}.
    }
    \label{fig:simple_sensitivity}
\end{figure}

\subsection{Simple Sensitivity Analysis}
Our method includes four hyperparameters:$\lambda_{position}$,$ \lambda_{distill}$,
$\lambda_{range}$, \textit{margin}. To facilitate comparisons, we apply our methods to Co2l with 200 examplars. The results are shown in \Cref{fig:simple_sensitivity}.Our approach consistently delivers robust outcomes across diverse hyperparameter values. Among the four hyperparameters, the \textit{margin} has a relatively larger impact on the results. The margin that is too small fails to ensure good inter-task separability, while the margin that is too large results in limited optimization regions within each task, thereby affecting the performance within tasks. Our empirical results show that a margin in the range of 0.1 to 0.15 ensures better performance in CIL.

\subsection{Additional Ablation Experiments}

\begin{itemize}
    \item (\Cref{appendix:additional_overlapping}) Feature overlap heatmap.
    \item (\Cref{sub:buffersize}) Ablation study with different buffer sizes.
    \item (\Cref{sub:tsne}) T-SNE visualization of features that demonstrate inter-task performance of our methods. 
    \item (\Cref{sub:sensitivity}) Extensive sensitivity analysis with different datasets in CIL and TIL scenarios.
    \item (\Cref{sec:threshold_loss}) Effectiveness of different loss function to regulate inter-class similarity.
    \item (\Cref{sec:time}) Analysis of the time cost.
\end{itemize}

\section{Conclusion, Limitations and Future Work}
\label{sec:conclusion}

In our work, we propose the \textit{global pre-fixing, local adjusting} strategy to tackle the feature confusion for contrastive continual learning. 
The strategy includes two steps: \textit{Global pre-fixing} pre-allocates non-overlapping regions for different tasks to avoid inter-task feature confusion. \textit{Local adjusting} optimizes the features localizing in pre-allocated region to mitigate intra-task feature confusion.
Our method achieves consistent performance improvement when integrated with existing contrastive continuous learning methods. However, our method still has several limitations: (1) Improvement in more complex datasets, such as TinyImageNet, is not so significant. (2) The threshold involved requires further selection. In the future, we hope to address these issues and explore their applications in online continual learning, self-supervised learning, and other broader real-world scenarios.



\section{Appendixes}
\subsection{Related Work}
\label{sec:related}

\subsubsection{Continual Learning}
The primary goal of continual learning is to enable machines to accumulate knowledge across multiple continuous tasks while minimizing the problem of forgetting. Existing methods can be categorized into three types. \textbf{Regularization-based methods}\cite{schwarz2018progress,zenke2017continual,aljundi2018memory,chaudhry2018riemannian,benzing2022unifying,liu2018rotate} use various kinds of regularization and constraints to preserve the knowledge of previous tasks. These methods can be broadly divided into two subcategories:
\textit{Parameter Weights Regularizations}\cite{schwarz2018progress,zenke2017continual,aljundi2018memory,chaudhry2018riemannian,benzing2022unifying,liu2018rotate,lee2020continual,park2019continual} estimate the importance of each parameter, such as ewc\cite{ewc} uses FIM\cite{benzing2022unifying}. And then they limit updates to the important ones.
\textit{Model-level Regularizations} \cite{li2017learning,ranne2017encoder,castro2018end} distill the knowledge from previous models (teachers) into current models (students) . The distilled knowledge can take the form of logits \cite{li2017learning,icarl}, features \cite{cha2021co2l,cassle,cclis,wen2024provable}, etc.  
\textbf{Replay-based methods} \cite{vitter1985random,lopez2017gradient,caccia2020online} integrate a small subset of past data into the current task to maintain performance. 
Some of them\cite{chaudhry2019tiny,riemer2018learning,vitter1985random,lopez2017gradient} directly store data from the past. \cite{kulesza2012determinantal,bang2021rainbow,kumari2022retrospective} explore how to store data to get better performance. Some of them\cite{shin2017continual,wu2018memory,ostapenko2019learning,xiang2019incremental} replay data through generative models. 
\textbf{Architecture-Based methods}\cite{wortsman2020supermasks,xue2022meta,kang2022forget} allocate different parameters for different tasks or introduce new model parameters\cite{du2023efficient,loo2020generalized,douillard2022dytox} for current task. 
For example, \cite{yan2021dynamically} propose DER, a two-stage approach that freezes the previously learned representation and augments it with additional feature dimensions for novel classes. In a similar vein, \cite{zhou2022model} present MEMO, a memory-efficient expandable model that adds specialized layers on top of a shared backbone to diversify representations with minimal memory cost. More recent studies leverage large pre-trained backbones: frozen pre-trained models already provide highly generalizable embeddings, and introducing adaptive tuning or modules (as in the APER framework) can further improve incremental plasticity and adaptivity\cite{zhou2025revisiting}. Together, these methods illustrate how dynamic architecture expansion and pretrained features help continual learners adapt to new classes while mitigating forgetting.
Most existing methods focus on classification performance, treating the model's encoder and classifier as a single unit, and applying the aforementioned strategies individually or in combination.

\subsubsection{Contrastive Continual Learning}

Existing methods primarily utilize contrastive loss in CL from two perspectives: (1) Some approaches exploit the transferability of features learned through contrastive loss to enhance CL performance. For example, \cite{cha2021co2l} combines supvised contrastive learning and relation knowledge distillation, achieving significant performance improvements. \cite{wen2024provable} provides a theoretical explanation for such performance enhancement\cite{cha2021co2l}. Based on \cite{cha2021co2l}, \cite{cclis} estimate previous task distributions via importance sampling with a weighted buffer during training. (2) Other approaches focus on incrementally building representations by self-supervised learning. These approaches have gained significant attention in recent years due to the widespread applications of self-supervised training in the large vision model. And it is impractical for a model to be trained on nearly infinite data all at once. The model must be able to learn continuously. In our work, we focus on the former. So far, most methods utilizing contrastive loss primarily focus on preventing feature forgetting. However, we argue that merely avoiding forgetting is insufficient for contrastive continual learning. More importantly, learning discriminative features across tasks is critical for aligning with the oracle model.

\renewcommand{\thefigure}{A\arabic{figure}}
\renewcommand{\thetable}{A\arabic{table}}
\setcounter{figure}{0}
\setcounter{table}{0}

\subsection{Details}
\subsubsection{Algorithm}
\label{sec:algorithm}
In the following, we detail the overall algorithm. Our algorithm can be integrated into any existing contrastive continual learning algorithms. We select \cite{cha2021co2l} as our baseline and highlight the additional steps of our approach using red colors. 
\begin{algorithm}[h]
\caption{Co2l+GPLASC}
\begin{algorithmic}[1]
\label{alg:algrithom_overview}

\REQUIRE Buffer size $B$, 
a sequence of training sets $\{D^t\}^T_{t=1}$.

\STATE  Initialize model $f_0$ and set buffer $\mathcal{M}\gets \emptyset$;
\STATE \textcolor{red}{
Initialize inter-task ETF with point set $\{P_{ETF}^1,P_{ETF}^2...P_{ETF}^T\}$;
}

\FOR{task $t = 1, \cdots, T$}

\STATE Construct dataset ${D} \gets D_t \cup \mathcal{M}$;
\STATE Initialize model $f^t\gets f^{t-1}$;

\STATE Compute ${\mathcal{L}}$ by ${\mathcal{L}} \gets \mathcal{L}_{Co2l}(f;B)$;
\\

\STATE \textcolor{red}{
Compute ${\mathcal{L}}$ by ${\mathcal{L}} \gets {\mathcal{L}} + \mathcal{L}_{range}+\mathcal{L}_{position}$;
} \\
\IF {$t > 1 $}
\STATE Update $ \mathcal{L}$ by \\ $ \mathcal{L} \gets  {\mathcal{L}} + \textcolor{red}{\mathcal{L}_{distill}}$;
\ENDIF

\STATE Update 
$f_{t}$ by SGD;

\STATE Collect buffer samples until $|\mathcal{M}|=B$;
\ENDFOR
\end{algorithmic}
\end{algorithm}
\vspace{-0.5em}

\subsubsection{Implementation details}
\label{sec:implementation}

We divide the training process into two stages. We use ResNet-18 (not pre-trained) as the base encoder for \textit{feature} learning, followed by a two-layer projection MLP that maps representations to a 128-dimensional \textit{latent feature} space. The hidden layer of the MLP projection consists of 512 hidden units.
During the representation learning phase, we apply a linear warm-up for the initial 10 epochs, followed by a cosine decay schedule for the learning rate, as proposed by \cite{loshchilov2016sgdr}. Training employs Stochastic Gradient Descent (SGD) with a momentum of 0.9 and a weight decay of 0.0001 across all trials.
For classifier training, we randomly select a class from the final period and uniformly sample an instance from that class. The linear classifier is trained over 100 epochs using SGD with a momentum of 0.9, omitting weight decay in this stage. The learning rate decays exponentially at the 60th, 75th, and 90th epochs, with decay rates of 0.2. We use learning rates of {0.5, 0.1, 0.05} for datasets {Seq-CIFAR-10, Seq-CIFAR-100, Seq-Tiny-ImageNet}, respectively. These training configurations are consistent with \cite{cha2021co2l}, ensuring a fair comparison.
\textbf{To obtain more transferable contrastive features, we calculate contrastive loss using \textit{latent features} and apply our method, including $\mathcal{L}_{R2SCL}$ and $\mathcal{L}_{distill}$, to the feature space before MLP.} Further explanation is provided in \Cref{sec:reason}.

\subsubsection{Conflicts with MLP}
\label{sec:reason}
MLP serves as a projection head to project features to another space. Through nonlinear transformation, MLP can extract more refined representations of input features and it becomes easier to distinguish positive samples from negative ones, enhancing the ability of supervised contrastive loss to group similar samples and differentiate dissimilar ones. Existing methods typically discard the MLP when training the linear classifier in the second stage. If we leverage our method to the \textit{latent features}, 
there would be a conflict that the restricted features are not consistent with the features employed in reality.
Moreover, the non-linear mapping makes these two spaces not equivalent. To resolve the conflict, we apply our method to the features before the MLP.

\subsubsection{Simple ETF Construction}
\label{sub:etf_construction}
A simplex equiangular tight frame can be produced from an orthogonal basis via:
	\begin{equation}\label{ETF_M}\nonumber
		\rmE=\sqrt{\frac{K}{K-1}}\mathbf{U}\left(\mathbf{I}_K-\frac{1}{K}\mathbf{1}_K\mathbf{1}_K^T\right),
	\end{equation}
where $\mathbf{E}=[\mathbf{e}_1,\cdots,\mathbf{e}_K]\in\mathbb{R}^{d\times K}$ is a simplex ETF, 
$\mathbf{U}\in\mathbb{R}^{d\times K}$ is an orthogonal basis and satisfies $\mathbf{U}^T\mathbf{U}=\mathbf{I}_K$,
$\mathbf{I}_K$ is an identity matrix and $\mathbf{1}_K$ is an all-ones vector.

\subsection{Proofs}
\subsubsection{Preliminary}

\begin{definition}[Gram Matrix \( \mathbf{G} \) of the Point Set]
Let \( S^{d-1} \subset \mathbb{R}^d \) represent the unit hypersphere and consider \( n \) points \( \{\mathbf{x}_1, \mathbf{x}_2, \dots, \mathbf{x}_n\} \) on \( S^{d-1} \). We define their gram matrix as:

\[
\mathbf{G} = [\mathbf{x}_i^\top \mathbf{x}_j]_{i,j=1}^n,
\]
where:
\begin{itemize}
    \item \( \mathbf{G}_{ii} = \mathbf{x}_i^\top \mathbf{x}_i = 1 \) (since the points lie on the unit hypersphere).
    \item \( \mathbf{G}_{ij} = k \) for \( i \neq j \).
\end{itemize}
\end{definition}

\begin{restatable}[\textbf{Conditions for Point Sets on a Hypersphere with Fixed Pairwise Inner Products}]{lem}{innercondition}
\label{lem:inner_condition}
Let \( S^{d-1} \) be a unit hypersphere in \( \mathbb{R}^d \), and consider \( n \) points \( \{\mathbf{x}_1, \mathbf{x}_2, \dots, \mathbf{x}_n\} \) on \( S^{d-1} \). For all pairs of distinct points, if their inner products are equal to a constant \( k \). For \( \mathbf{G} \) to represent a valid point configuration, the following conditions must be hold:  

\noindent 1.Range of \( k \):
\[
-\frac{1}{n-1} \leq k \leq 1.
\]
2. Dimensional constraint
\[d \geq n-1\] 
\end{restatable}

\begin{proof}
    
    \( \mathbf{G} \) can be expressed as:
        \[
    \mathbf{G} = (1-k)\mathbf{I}_n + k\mathbf{1}_n\mathbf{1}_n^\top,
    \]
     Where \( \mathbf{I}_n \) is the \( n \times n \) identity matrix, and \( \mathbf{1}_n \) is an \( n \)-dimensional vector of ones.
    For \( \mathbf{G} \) to represent a valid point configuration, it must be positive semidefinite (PSD). 
    
    Thus, for \( \mathbf{G} \) to represent a valid point configuration, it must be positive semidefinite (PSD). The eigenvalues of \( \mathbf{G} \) are:
    \[
    \lambda_1 = 1 + k(n-1), \quad \lambda_2 = 1-k \; \text{(with multiplicity \( n-1 \))}.
    \]
    The PSD condition imposes the following constraints:
    \[
    1 - k \geq 0 \quad \text{and} \quad 1 + k(n-1) \geq 0.
    \]
    
Combining the above inequalities gives range of \( k \):
\[
-\frac{1}{n-1} \leq k \leq 1.
\]
    The ambient dimension \( d \) must satisfy \( d \geq n-1 \). This ensures that the \( n \) points can be embedded in \( \mathbb{R}^d \) without violating the inner product constraint.

\end{proof}

\begin{remark}[Interpretation of \( k \)] When $k$ takes different values, the geometric distribution of our features will vary:
\begin{itemize}
    \item \( k = 1 \): All points coincide.
    \item \( k = 0 \): Points are pairwise orthogonal (possible only for \( n \leq d+1 \)).
    \item \( k = -\frac{1}{n-1} \): Points form a spherical simplex, such as a regular tetrahedron for \( n = 4 \).
\end{itemize}

\end{remark}

\subsection{\Cref{lemma:local_supcon_theorem}}
\label{appendix:proofs}
We first prove that the final features geometry is a $\rho$-sphere-inscribed regular simplex in Part 1 and we prove the lower bound of inter-class similarity constrained constrastive loss in Part 2.

\localsupcontheorem*

\noindent \textbf{PART1:}

\begin{proof}[proof]

   \noindent The supervised contrasive loss $\LSC(Z;Y)$ has a lower bound based on \Cref{lemma:lower_bound_contrastive}:
    \begin{align}
        \LSC(Z;Y) \notag
        \ge
        &\sum_{l = 2}^{b} l M_l  \notag                          
            \log 
            ( 
                l - 1 + (b-l)
                \exp ( 
                    \frac {1} {M_l}
                    ( \\
                        &\underbrace{ \notag
                             - \frac{1}{|B_{y}|\left(|B_{y}|-1\right)}\sum_{y\in \mathcal{Y}}{\sum_{B\in 
                            \mathcal{B}_{y,l}}{
                                \sum_{i\in B_{y}}\sum_{j\in B_{y}\setminus\{\{i\}\}}\langle z_{i},z_{j}\rangle 
                            }}
                        }_{\text{attraction term}}\\
                        &+ 
                        \underbrace{
                            ( 
                            \frac{|\mathcal{B}_{y,l}|}{N^2}\frac{|\mathcal{Y}|^2}{|\mathcal{Y}|-1}
                            \sum_{y\in\mathcal{Y}}\sum_{ \substack{n\in[N]\\y_n=y}}
                            \sum_{y_m\neq y}\langle z_n,z_m \rangle 
                            ) 
                        }_{\text{repulsion term}}
                    )
                )
            )\nonumber
        \enspace,
    \end{align}

    For \textit{attraction term}, when features within the same class collapse to a point i.e. $\langle z_{i},z_{j}\rangle =1$ where $y_{i}=y_{j}$. We let each $z \in y_i$ converging to a single vector $\zeta_i$. Our assumption has no impact on the final result.
    
    For the \textit{ repulsion term}, we ignore the constant terms and focus on \( \sum_{y \in \mathcal{Y}} \sum_{\substack{n \in [N] \\ y_n = y}} \sum_{y_m \neq y} \langle z_n, z_m \rangle \). Such term represents all negtive pairs in the batch. 

    \begin{align}
        &min \sum_{y_i \neq y_j}\left< z_i,z_j \right>\nonumber \\
        s.t. &\forall y_i \neq y_j ,\left< z_i,z_j \right> \ge k\nonumber
    \end{align}

        \begin{align}
    \mathcal{L}\left( \left< z_i,z_j \right>,\lambda  \right) =&\sum_{y_i \neq y_j}\left< z_i,z_j \right>+\sum_{y_i \neq y_j} \lambda_{ij}\left(  k-\left< z_i,z_j \right>\right)
    \nonumber
    \end{align}
    According to the KKT conditions, we obtain.
    \begin{align}
        &\left\{ \begin{array}{l}
            \forall y_i \neq y_j,1-\lambda_{ij}=0\\
            \lambda_{ij}\ge0\\
            \forall y_i \neq y_j, k-\left< z_i,z_j \right> \leq 0 \nonumber\\
             \lambda_{ij}\left( k-\left< z_i,z_j\right> \right)=0    
        \end{array} \right. \\
        \Rightarrow 
        &\forall  y_i \neq y_j ,\left< z_i,z_j\right>=k \nonumber
    \end{align}

    For each negative sample pair \( \left< z_i, z_j \right> \), it a convex optimization process. The resulting KKT conditions provide a necessary and sufficient condition\footnote{The Slater condition is clearly satisfied.}. This conclusion leads to a very insightful result: when the inner products between all complex numbers reach the threshold, \textit{repulsion term} is minimized. Next, we will explore the conditions and corresponding feature structure when the  \( \left< z_i, z_j \right> =k \) is achieved. 
    We show the \Cref{lem:inner_condition} along with its proof. Note that when $d \geq |\mathcal{Y}| - 1$ and $-\frac{1}{|\mathcal{Y}|-1} \leq k \leq 1$ hold, we can achieve \( \left< z_i, z_j \right> =k \), for \( y_i \neq y_j\). In particular, when \(  k = -\frac{1}{|\mathcal{Y}|-1} \), we exactly form a $\rho_{\mcZ}$-sphere-inscribed regular simplex, which corresponds to the conclusion in the \cite{graf2021dissecting}. However, when \(  k > -\frac{1}{|\mathcal{Y}|-1} \), our minimum value is always directly related to the threshold $k$, providing us with an approach to control the distribution within tasks.

Next, we explore the final geometric structure formed under the constraint of similarity. 

    Assume that for any $i,j$, there exist $z_i,z_j$, $\left< z_i,z_j \right> = k $.\\
1.\textit{Centroid Condition}
    
    \begin{equation}
        \sum_i (z_i-\frac{1}{n} \sum_j z_j)=0 \nonumber
    \end{equation}

\noindent 2.\textit{Radius Condition} 
          
          for i $\in$ $[N]$ 
          
  \begin{align}
    \rho^2 =&||z_i-\frac{1}{N} \sum_j z_j||^2 \nonumber\\
        =&||z_i||^2+\frac{1}{N^2} \sum_{i,j} \left< z_i,z_j\right> -2 z_i\frac{1}{N} \sum_j z_j \nonumber\\
        =&1+\frac{1}{N^2} ( (\frac{N}{|\mcY|})^2 \cdot |\mcY|+(|\mcY|-1)|\mcY|\cdot (\frac{N}{|\mcY|})^2 \cdot k) \notag \nonumber\\ &
        - 2\frac{1}{N} (\frac{N}{|\mcY|}+(N-\frac{N}{|\mcY|})\cdot k) \nonumber\\
        =& (1-\frac{1}{|\mcY|})(1-k)\nonumber
    \end{align}
3.\textit{Equiangular Condition} $\exists k \in \R: k=<\zeta_i,\zeta_j>$, 

From 1.2.3, we can conclude that the result will form a regular simplex inscribed within a hypersphere, centered at \( \frac{1}{n} \sum_j z_j) \) with radius \(  \sqrt{(1-\frac{1}{|\mcY|})(1-k)} \).

\end{proof}

\noindent\textbf{PART2:}

Before we prove the new lower bound of task similarity constrained constrastive loss, we review some theory results in\cite{graf2021dissecting}, which simplifies our proof.

\begin{lemma}[Sum of attraction terms]
    \label{lem:supcon_att}
    Let $l\in \set{2,\dots,b}$ and let $\mcZ = \bbS_{\rho_{\mathcal Z}}$. 
    For every $Y\in \mcY^N$ and every $Z\in \mcZ^N$, it holds that 
    \begin{align}
        \label{lem:supcon_att:eq1}
        &- \frac{1}{|B_{y}|\left(|B_{y}|-1\right)}\sum_{y\in \mathcal{Y}}{\sum_{B\in 
                            \mathcal{B}_{y,l}}{
                                \sum_{i\in B_{y}}\sum_{j\in B_{y}\setminus\{\{i\}\}}\langle z_{i},z_{j}\rangle 
                            }} \notag \\
        \ge
        &- \left(\sum_{y\in\mcY} |\mcB_{y,l}|\right) 
        {\rho_{\mcZ}}^2
        \enspace,
    \end{align}
    where equality is attained if and only if:
        \label{con:supcon_att}
            For every $n,m\in [N]$, $y_n = y_m$ implies $z_n = z_m$\enspace.
\end{lemma}

\begin{remark}
proof in \cite{graf2021dissecting}. According to \Cref{lem:supcon_att}, \textit{attraction term} has a lower bound.This means that the features in the same class need to collapse to a single point to ensure that we can achieve the lower bound.
Therefore, the key lies in the \textit{repulsion term}, which determines how we can push samples from different classes further.
\end{remark}

\begin{proof}[proof(part2)]
\begin{align}
        &\LSC(Z;Y) \notag
        \ge\\
        &\sum_{l = 2}^{b} l M_l \notag
            \log 
            ( 
                l - 1 + (b-l)
                \exp ( 
                    \frac {1} {M_l}
                    ( \\
                        &\underbrace{ \notag
                             - \frac{1}{|B_{y}|\left(|B_{y}|-1\right)}\sum_{y\in \mathcal{Y}}{\sum_{B\in 
                            \mathcal{B}_{y,l}}{
                                \sum_{i\in B_{y}}\sum_{j\in B_{y}\setminus\{\{i\}\}}\langle z_{i},z_{j}\rangle 
                            }}
                        }_{\text{attraction term}}\\
                        &+
                        \underbrace{
                            \frac{|\mathcal{B}_{y,l}|}{N^2}\frac{|\mathcal{Y}|^2}{|\mathcal{Y}|-1}
                            \sum_{y\in\mathcal{Y}}\sum_{ \substack{n\in[N]\\y_n=y}}
                            \sum_{y_m\neq y}\langle z_n,z_m \rangle 
                        }_{\text{repulsion term}}
                    )
                )
            )\nonumber
        \enspace,
    \end{align}
   
We obtained the lower bound of the \textit{attraction term} from the \Cref{lem:supcon_att}.

When \( d \geq |\mcY| - 1 \) and \( -\frac{1}{|\mcY|-1} \leq k \leq 1 \), each negtive pairs can reach its minimum.The \textit{repulsion term} lower bound can be:

\begin{align}
    &\sum_{y\in \mathcal{Y}}{\sum_{B\in \mathcal{B}_{y,l}}{S}_{\text{rep}}}\left( Z;Y,B,y \right)\nonumber\\ 
   \ge& 
      \textit{Constant} \cdot \left(
        |\mcY| \cdot \frac{N}{|\mcY|} \cdot 
        (N-\frac{N}{|\mcY|}) \cdot k
    \right)\cdot {\rho_{\mcZ}}^2  \nonumber\\
    =&
    \frac{|\mcB_{y,l}|}{\mult{Y}(y)(N - \mult{Y}(y))} \cdot \left(
        N^2 \cdot (1-\frac{1}{|\mcY|}) \cdot k
    \right)\cdot {\rho_{\mcZ}}^2 \nonumber\\
    =&\frac{|\mcB_{y,l}|}{N^2} \frac{|\mcY|^2}{|\mcY|-1}\cdot  \left(
        N^2 \cdot (1-\frac{1}{|\mcY|}) \cdot k
    \right)\cdot {\rho_{\mcZ}}^2 \nonumber\\
    =&
    |\mcB_{y,l}|  |\mcY| k\cdot {\rho_{\mcZ}}^2 
    \label{eqn:lower_bound_rep}
\end{align}

 Leveraging the bounds on the \textit{attraction terms} and the \textit{repulsion terms} from \cref{lem:supcon_att:eq1} and \cref{eqn:lower_bound_rep}, respectively, we get

\begin{align}
         & \textit{attraction terms} + \textit{repulsion terms} \notag \\
         \ge &
            - |\mcY| |\mcB_{y,l}|
            {\rho_{\mcZ}}^2
            + 
             |\mcY||\mcB_{y,l}| k {\rho_{\mcZ}}^2 
            \nonumber\\
        = & 
            - |\mcY| |\mcB_{y,l}|
            {\rho_{\mcZ}}^2
            (1-k) 
        \enspace.
        \label{eqn:lower_bounde_both}
    \end{align}
 Since $Y$ is balanced, $|\mcB_{y,l}|$ does not depend on $y$, and so 
    \begin{equation}
    \label{eqn:Ml}
        M_l = {\sum_{y\in Y} |\mcB_{y,l}|} = |\mcY| |\mcB_{y,l}|
        \enspace.
    \end{equation}

\Cref{eqn:Ml} +\Cref{eqn:lower_bounde_both} 
    \[\begin{array}{c}
    {{\cal L}_{{\rm{SC}}}}(Z;Y) \ge \sum\limits_{l = 2}^b l {M_l}\log (l - 1 + (b - l)\exp (
     - (1-k){\rho_{\mcZ}}^2
             ))
    \end{array}\]

\end{proof}

\subsection{When \(  -1 \leq k \leq -\frac{1}{|\mcY|-1} \)}
\label{sec:k_annother}

    We show the final results in \cite{graf2021dissecting}, when $k \geq -1$:
    \begin{restatable}[\textbf{Supervised contrastive loss lower bound}]{corollary}{thm@supcon}
    \label{thm:supcon_continual}
    Let $\rho_{\mathcal Z}>0$ , let $\mcZ = \bbS_{\rho_{\mathcal Z}}^{h-1}$ 
    Further,for each task $t$,let $Z=(z_1,\ldots,z_N) \in \mcZ^N$ be an $N$ point configuration with labels $Y=(y_1,\ldots,y_N) \in [|\mcY|]^N$. If the label configuration $Y$ is balanced, it holds that 
    \begin{align*}
        \LSC(Z;Y) 
        \ge 
        \sum_{l=2}^{b} 
        l\, M_l
        \log 
        \left( 
            l - 1 + (b-l)
            \exp \left( 
                - \frac{|\mcY|\rho_{\mcZ}^2}{|\mcY|-1}                 
            \right)
        \right)\enspace,
    \end{align*}
    where 
    \begin{equation*}
        M_l = \sum_{y\in\mcY} |\set{ {B \in \mcB}:~ |B_y|=l }|\enspace.
    \end{equation*}
    Equality is attained if and only if the following conditions are satisfied. There are $\zeta_1, \dots, \zeta_{K} \in \R^h$ such that:
    \begin{enumerate}
        \item
        $\forall n \in [N]: z_n = \zeta_{y_n}$
        \item 
        $\{\zeta_y\}_y$ form a $\rho_{\mcZ}$-sphere-inscribed regular simplex
    \end{enumerate}
\end{restatable}

Note that when $k \geq -1$, the minimum value is achieved when $ k=-\frac{1 }{|\mcY|-1}$, which implies that when \(  -1 \leq k \leq -\frac{1}{|\mcY|-1} \), the minimum is achieved when $ k=-\frac{1}{|\mcY|-1}$ (the global minimum is smaller than the local minimum).

\textbf{It also shows that \cite{graf2021dissecting} is our special case.}

\subsection{Threshold Selection}

\begin{figure}[H]
	\centering
    \includegraphics[width=0.5\linewidth]{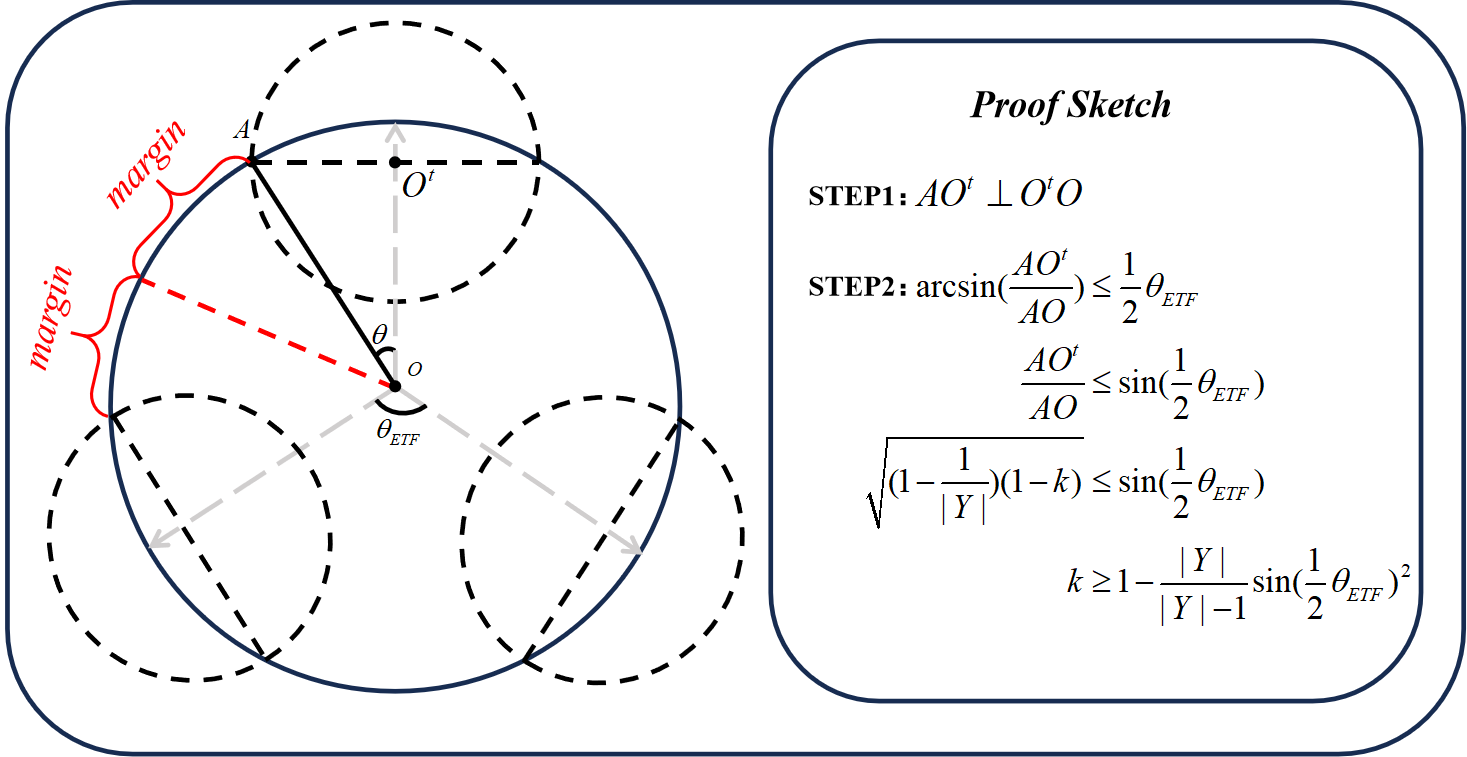}
	\caption{An illustration and a simple proof sketch for threshold selection. $O^t$ represents the origin of the final inscribed hypersphere for task $t$. Due to the symmetry of the pre-allocated ETF, the thresholds for different tasks are the same.}
	\label{fig:threshold_proof}
\end{figure}

\begin{proof}
    STEP1: For each task, according to our theoretical work mentioned earlier, features with the same label $y_i$ will collapse to a specific vector $\zeta_{y_i}$, and the final feature will form a regular simplex inscribed within a hypersphere, centered at \( \frac{1}{n} \sum_j z_j \) with radius \(  \sqrt{(1-\frac{1}{|\mcY|})(1-k)} \). In other words, our objective is to prove that: 
    \[
(\zeta_{i}- \frac{1}{|\mcY|} \sum_j \zeta_{j} )  \bot \frac{1}{|\mcY|} \sum_j \zeta_{j} 
    \]
    where $\zeta_{i} \in \{\zeta_{1},...,\zeta_{|\mcY|} \}$. if $|\mcY|=2$ and the feature dimension is 2, we show it in \Cref{fig:threshold_proof}. $AO$ represents a specific $\zeta_{i}$.
    \begin{align}
        &(\zeta_{i}- \frac{1}{|\mcY|} \sum_j \zeta_{j} ) \cdot \frac{1}{|\mcY|} \sum_j \zeta_{j}\nonumber\\
        =&\zeta_{i} \cdot\frac{1}{|\mcY|} \sum_j \zeta_{j}-(\frac{1}{|\mcY|} \sum_j \zeta_{j})^2\nonumber\\
        =& \zeta_{i} \cdot\frac{1}{|\mcY|} \sum_j \zeta_{j}-\frac{1}{|\mcY|^2}(\sum_{j} ||\zeta_{j}||^2+ 2 \cdot \sum_{1 \leq j < l  \leq |\mcY|} \zeta_{j} \cdot \zeta_{l}  )\nonumber\\
        =& \frac{1}{|\mcY|} \cdot (1+(|\mcY|-1)\cdot k)-\frac{1}{|\mcY|^2}(\sum_{j} 1+ 2 \cdot \binom{n}{2} \cdot k
        )\nonumber\\
        =& \frac{1}{|\mcY|} \cdot (1+(|\mcY|-1)\cdot k)-\frac{1}{|\mcY|^2}(|\mcY|+ |\mcY|(|\mcY|-1) \cdot k
        )\nonumber\\
        =&\frac{1}{|\mcY|} \cdot (1+(|\mcY|-1)\cdot k-(1+(|\mcY|-1)k))\nonumber\\
        =&0\nonumber
    \end{align}

\noindent STEP2: We denote AO as a specific $\zeta_{i}$:
    
\begin{align}
    arcsin(\frac{AO^t}{AO}) &\leq \frac{1}{2} \theta_{ETF}\nonumber\\
    \frac{AO^t}{AO} &\leq sin(\frac{1}{2}\theta_{ETF})\nonumber\\
    \sqrt{(1-\frac{1}{|\mcY|})(1-k)}&\leq  sin(\frac{1}{2}\theta_{ETF})\nonumber\\
    k &\geq1-\frac{|\mcY|}{|\mcY|-1}sin(\frac{1}{2}\theta_{ETF})^2\nonumber
\end{align}

\end{proof}

\subsection{Hyperparameters}
\subsubsection{Hyperparameters Selection}

\label{sec:parameters_selection}

As shown in \Cref{tab: hyper_selection},
we conduct a thorough hyperparameter exploration, considering the following key hyperparameters: $\lambda_{position}$, $\lambda_{distill}$, $\lambda_{range}$, \textit{margin}, learning rate ($lr$), number of epochs ($epochs$), number of epochs in the first task (\textit{start\_epoch}) and batchsize ($bsz$). Similar explorations are carried out for all baselines to ensure a consistent comparison. For parameters not explicitly mentioned, we adhere to the optimal values provided by the original baseline to maintain alignment with its configuration.
\subsubsection{Final Hyperparameters}
As shown in \Cref{tab: final_hyper_selection}, except for the hyperparameters specific to our work, all other hyper-parameters are set to the optimal values as provided by the original baseline. This ensures a fair comparison by isolating the effects of our GPLASIC mechanism while maintaining consistency with the baseline's established settings.
\label{appendix:hyper-parameters}

\begin{table*}[t]
    \centering
    \small
    \caption{\textbf{Hyper-parameters selection in our experiments.}}
    \label{tab:final_hyper_selection}
    \begin{tabularx}{\textwidth}{ccX} 
        \hline
        \textbf{Method} & \textbf{Buffer} & \textbf{Hyper-parameters} \\
        \hline
        \multicolumn{3}{c}{\textbf{Seq-CIFAR-10}} \\
        \hline
        \multirow{2}{*}{Co2l+GPLASC} & 200 & $\tau:0.1, \lambda_{distill}:1, \lambda_{position}:1, \lambda_{range}:1, margin:0.15, epochs:100, start\_epochs:500$, \textit{bsz}:256 \\
                                     & 500 & \\
        \multirow{2}{*}{CCLIS+GPLASC} & 200 & $\tau:0.1, \lambda_{distill}:1, \lambda_{position}:1, \lambda_{range}:1, margin:0.15, epochs:100, start\_epochs:500$, \textit{bsz}:256 \\
                                     & 500 & \\
        \hline
        \multicolumn{3}{c}{\textbf{Seq-CIFAR-100}} \\
        \hline
        \multirow{2}{*}{Co2l+GPLASC} & 200 & $\tau:0.1, \lambda_{distill}:1, \lambda_{position}:1, \lambda_{range}:1, margin:0.1, epochs:100, start\_epochs:500$, \textit{bsz}:256 \\
                                     & 500 & \\
        \multirow{2}{*}{CCLIS+GPLASC} & 200 & $\tau:0.1, \lambda_{distill}:1, \lambda_{position}:1, \lambda_{range}:1, margin:0.1, epochs:100, start\_epochs:500$, \textit{bsz}:256 \\
                                     & 500 & \\
        \hline
        \multicolumn{3}{c}{\textbf{Seq-Tiny-ImageNet}} \\
        \hline
        \multirow{2}{*}{Co2l+GPLASC} & 200 & $\tau:0.1, \lambda_{distill}:1, \lambda_{position}:1, \lambda_{range}:1, margin:0.1, epochs:50, start\_epochs:500$, \textit{bsz}:256 \\
                                     & 500 & \\
        \multirow{2}{*}{CCLIS+GPLASC} & 200 & $\tau:0.1, \lambda_{distill}:1, \lambda_{position}:1, \lambda_{range}:1, margin:0.1, epochs:50, start\_epochs:500$, \textit{bsz}:256 \\
                                     & 500 & \\
        \hline
        \multicolumn{3}{c}{\textbf{CUB200}} \\
        \hline
        \multirow{2}{*}{Co2l+GPLASC} & 200 & $\tau:0.1, \lambda_{distill}:1, \lambda_{position}:1, \lambda_{range}:1, margin:0.1, epochs:200, start\_epochs:500$, \textit{bsz}:128 \\
                                     & 500 & \\
        \hline
        \multicolumn{3}{c}{\textbf{ImageNet-R}} \\
        \hline
        \multirow{2}{*}{Co2l+GPLASC} & 200 & $\tau:0.1, \lambda_{distill}:1, \lambda_{position}:1, \lambda_{range}:1, margin:0.1, epochs:200, start\_epochs:500$, \textit{bsz}:128 \\
                                     & 500 & \\
        \hline
        \multicolumn{3}{c}{\textbf{ImageNet-A}} \\
        \hline
        \multirow{2}{*}{Co2l+GPLASC} & 200 & $\tau:0.1, \lambda_{distill}:1, \lambda_{position}:1, \lambda_{range}:1, margin:0.1, epochs:200, start\_epochs:500$, \textit{bsz}:128 \\
                                     & 500 & \\
        \hline
    \end{tabularx}
\end{table*}

\begin{table*}[!htbp]
    \centering  
    \begin{tabular}{ccc}
	\hline
\textbf{Dataset}&\textbf{Parameter}&\textbf{values}\\
\hline
\multirow{8}{*}{Seq-Cifar-10}
&\textit{margin}&{0, 0.05, 0.1, 0.15, 0.2, 0.3, 0.5}\\
&$\lambda_{distill}$&{0.1, 0.2, 0.3, 0.5,1,2,3}\\
&$\lambda_{range}$&{0.5, 1, 2, 3, 4, 5}\\
&$\lambda_{position}$&{0.5, 0.8, 1, 2, 3}\\
&learning rate&{0.1, 0.2, 0.3, 0.5,1}\\  
&$start\_epochs$&{500,200}\\
&$epochs$&{100, 50}\\
&$bsz$&{256, 512, 1024}\\
&$\tau$&{0.1, 0.5, 1.0}\\
\hline

\multirow{8}{*}{Seq-Cifar-100}
&\textit{margin}&{0, 0.05, 0.1, 0.15, 0.2, 0.3, 0.5}\\
&$\lambda_{distill}$&{0.1, 0.2, 0.3, 0.5,1,2,3}\\
&$\lambda_{range}$&{0.5, 1, 2, 3, 4, 5}\\
&$\lambda_{position}$&{0.5, 0.8, 1, 2, 3}\\
&learning rate&{0.1, 0.2, 0.3, 0.5,1}\\
&$start\_epochs$&{500,200}\\
&$epochs$&{100, 50}\\
&$bsz$&{256, 512, 1024}\\
&$\tau$&{0.1, 0.5, 1.0}\\
\hline  

\multirow{8}{*}{
Seq-Tiny-Imagenet/CUB200
/Imagenet-R
/Imagenet-A}
&\textit{margin}&{0, 0.05, 0.1, 0.15}\\
&$\lambda_{distill}$&{0.1, 0.2, 0.3, 0.5}\\
&$\lambda_{range}$&{0.5, 1, 2}\\
&$\lambda_{position}$&{1, 2, 3}\\
&learning rate&{0.5}\\
&$start\_epochs$&{500}\\
&$epochs$&{50,200}\\
&$bsz$&{64,128,256,512}\\
&$\tau$&{0.1, 0.5, 1.0}\\
\hline

\end{tabular}
    \caption{
    \textbf{Hyperparameter Space}}
\label{tab: hyper_selection}
\end{table*}

\subsection{Additional Experiments}

\label{appendix:additional_experiements}

\subsubsection{\textit{Average Forgetting} Evaluations}
We utilize the Average Forgetting metric as defined following to quantify how much information the model has forgotten about previous tasks:
\[
F=\frac{1}{T-1} \sum_{i=1}^{T-1} max_{t\in \{ 1,...,T-1 \}} (Acc_{t,i}-Acc_{T,i})
\]
\Cref{tab:forgetting_result} report the average forgetting results of our method compared to all other contrastive continual learning methods. The results show that our method can effectively mitigate forgetting, especially even without using additional buffers.

\begin{table*}
    \centering
    \small
    \begin{tabular}{cccccccc}
	\hline
\multirow{2}{*}
{\textbf{Buffer}}&\textbf{Dataset}&\multicolumn{2}{c}{\textbf{Seq-Cifar-10}}&\multicolumn{2}{c}{\textbf{Seq-Cifar-100}}&\multicolumn{2}{c}{\textbf{Seq-Tiny-ImageNet}}\\
&\textbf{Scenario}&\textbf{Class-IL}&\textbf{Task-IL}&\textbf{Class-IL}&\textbf{Task-IL}&\textbf{Class-IL}&\textbf{Task-IL}\\
\hline
\multirow{6}{*}{200}
&Co2l
&36.35$\pm$1.16
&6.71$\pm$0.35
&67.82$\pm$0.41
&38.22$\pm$0.34
&73.25$\pm$0.21
&47.11$\pm$1.04
\\ 

&Co2l + ours
&\textbf{25.43}$\pm$\textbf{1.35}
& \textbf{5.37}$\pm$\textbf{0.85}
&\textbf{59.64}$\pm$\textbf{1.37}
&\textbf{36.58}$\pm$\textbf{0.58}
&\textbf{68.64}$\pm$\textbf{1.54}
&\textbf{42.71}$\pm$\textbf{0.47}
\\ 
\cmidrule[0.5pt]{2-8} 
&CILA 

&30.03$\pm$1.40 
&\textbf{4.66$\pm$0.55}
&63.30$\pm$2.26
&35.37$\pm$0.54
&70.63$\pm$2.28
&39.34$\pm$0.81
\\ 
&CILA + ours
&\textbf{23.96}$\pm$\textbf{1.57}
&5.57$\pm$0.84
&\textbf{51.38$\pm$2.18}
&\textbf{33.50$\pm$0.58}
&\textbf{65.44$\pm$1.83}
&\textbf{36.77$\pm$0.78}
\\ 
\cmidrule[0.5pt]{2-8}
&CCLIS
&22.59$\pm$0.18
&\textbf{2.08$\pm$0.27}
&46.89$\pm$0.59
&14.17$\pm$0.20
&62.21$\pm$0.34
&\textbf{33.20$\pm$0.75}
\\

&CCLIS + ours
&\textbf{21.07$\pm$2.28}
&2.19$\pm$0.47
&\textbf{43.82$\pm$1.74}
&\textbf{12.32$\pm$0.54}
&\textbf{60.58$\pm$1.30}
&33.33$\pm$0.82
\\
\hline
\multirow{6}{*}{500}
&Co2l
&25.33$\pm$0.99
&\textbf{3.41$\pm$0.8}
&51.96$\pm$0.80
&26.89$\pm$0.45
&65.15$\pm$0.26
&39.22$\pm$0.69
\\
&Co2l + ours
&\textbf{22.83$\pm$2.01}
&4.56$\pm$0.87
&\textbf{43.69$\pm$1.22}
&\textbf{26.67$\pm$1.22}
&\textbf{60.26$\pm$1.60}
&\textbf{35.75$\pm$0.59}
\\ 
\cmidrule[0.5pt]{2-8} 
&CILA
&\textbf{23.05$\pm$2.45}
&\textbf{4.49$\pm$0.59 }
&\textbf{46.24$\pm$1.17}
&\textbf{23.68$\pm$1.23}
&60.95$\pm$1.41
&37.65$\pm$0.91
\\ %

&CILA + ours
&25.37$\pm$1.40 
&4.62$\pm$1.44
&49.29$\pm$1.59
&30.08$\pm$1.15
&\textbf{60.06$\pm$1.43}
&\textbf{35.30$\pm$1.77}
\\ 
\cmidrule[0.5pt]{2-8} 
&CCLIS
&18.93$\pm$0.61
&\textbf{1.69$\pm$0.12}
&42.53$\pm$0.64
&12.68$\pm$1.33
&50.15$\pm$0.20
&\textbf{23.46$\pm$0.93}
\\ 
&CCLIS + ours
&\textbf{17.30$\pm$1.42}
&3.68$\pm$0.72
&\textbf{40.53$\pm$1.81}
&\textbf{11.54$\pm$0.70}
&\textbf{49.15$\pm$1.90}
&25.30$\pm$0.96
\\
\hline
\end{tabular}
    \caption{Average Forgetting (lower is better) in Continual Learning, all averaged across five independent trails. For simplicity, we only compare our method with contrastive continual learning methods}
\label{tab:forgetting_result}
\end{table*}

\subsubsection{More experiments about feature overlapping}
\label{appendix:additional_overlapping}
We select ResNet-18 with feature dimension of 512 as our encoder and complete our experiments on CIFAR-10. The CIFAR-10 dataset is divided into 5 tasks, each containing 2 classes. We choose SupCon(simply fine-tuning on all tasks with 200 examplars), Co2l\cite{cha2021co2l}, our methods, and oracle model and compare the \textit{overlapping} between these four methods. 
We define \textit{overlapping} between two distributions as:
\[
E_O(P_A,P_B)=\int_{R{^d}} {\min (P_A(x),P_B(x))dx} 
\]
$P_A$ and $P_B$ can be the feature distribution for a specific class or task. 
Specifically, after finishing all tasks, we used the validation set of each task to estimate the probability density functions via kernel density estimation (KDE). 
On the one hand, considering the feature dimension, the time complexity in calculating the KDE for N samples in a space of dimensions D is \(O(N^2D)\). When the feature dimension is very large, the time cost increases significantly. On the other hand, due to the curse of dimensionality, to avoid the sparsity in high-dimensional spaces, we first apply PCA to reduce the dimensionality to a 3D subspace before performing the KDE calculation.
Note that the features were not mapped onto a hypersphere. Then, we calculate the degree of overlap $E_O(P_A, P_B)$ between the tasks and visualized the results as a heatmap, as shown in \Cref{fig:resnet18_overlapping}.

\begin{table*}[!htbp]
    \centering
    \begin{tabular}{cccccc}
        \hline  
        &\textbf{100}&\textbf{200}&\textbf{300}&\textbf{400}&\textbf{500}\\
        \hline
           SupCon &48.82$\pm1.08$&51.97$\pm$0.85&53.12$\pm$0.98&55.99$\pm$1.53&57.94$\pm$2.15\\
            Co2l &60.68$\pm$1.38&65.57$\pm$1.37&71.41$\pm$1.58&72.38$\pm$1.54&$74.26\pm$1.42\\
           \textbf{ Co2l + ours} &\textbf{64.63}$\pm$\textbf{1.68}&\textbf{69.9}$\pm$\textbf{1.26}&\textbf{73.58}$\pm$\textbf{2.01}&\textbf{75.9}$\pm$\textbf{0.98}&\textbf{77.24}$\pm$\textbf{2.52}\\
        \hline
    \end{tabular}
    \caption{Different buffer sizes on Seq-Cifar-10 with Resnet18 (CIL).}
    \label{tab:buffer_abla_cil}
\end{table*}

\begin{figure*}[t]
	\centering
	\subfigure[SupCon]{
		\begin{minipage}[b]{0.2\textwidth}
		\includegraphics[width=1\textwidth]{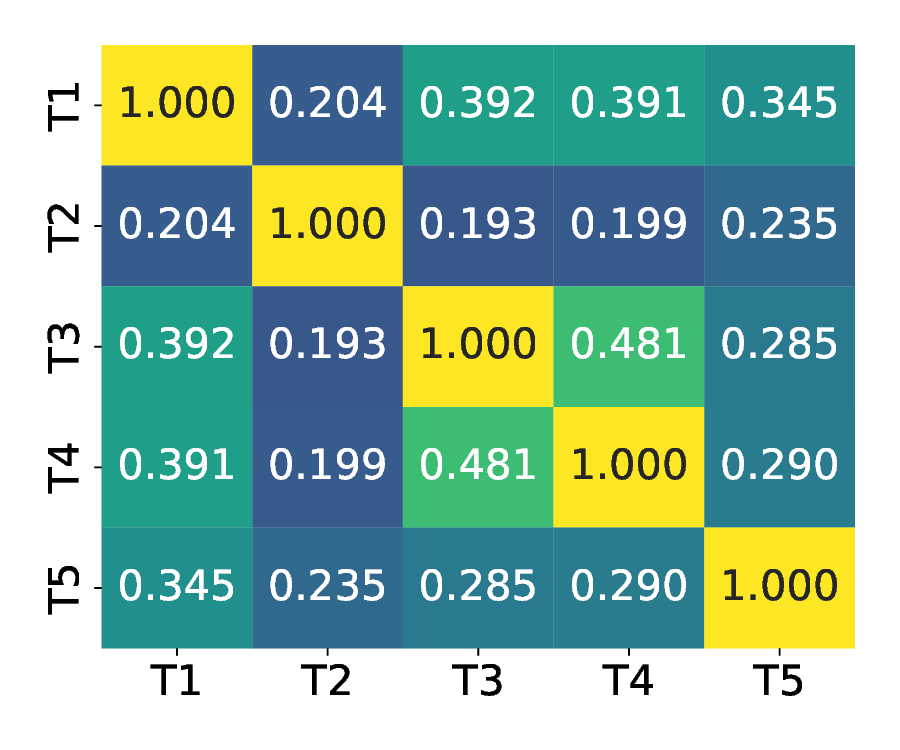}
		\end{minipage}
       
	}
    \subfigure[Co2l]{
	   \begin{minipage}[b]{0.2\textwidth}
		\includegraphics[width=1\textwidth]{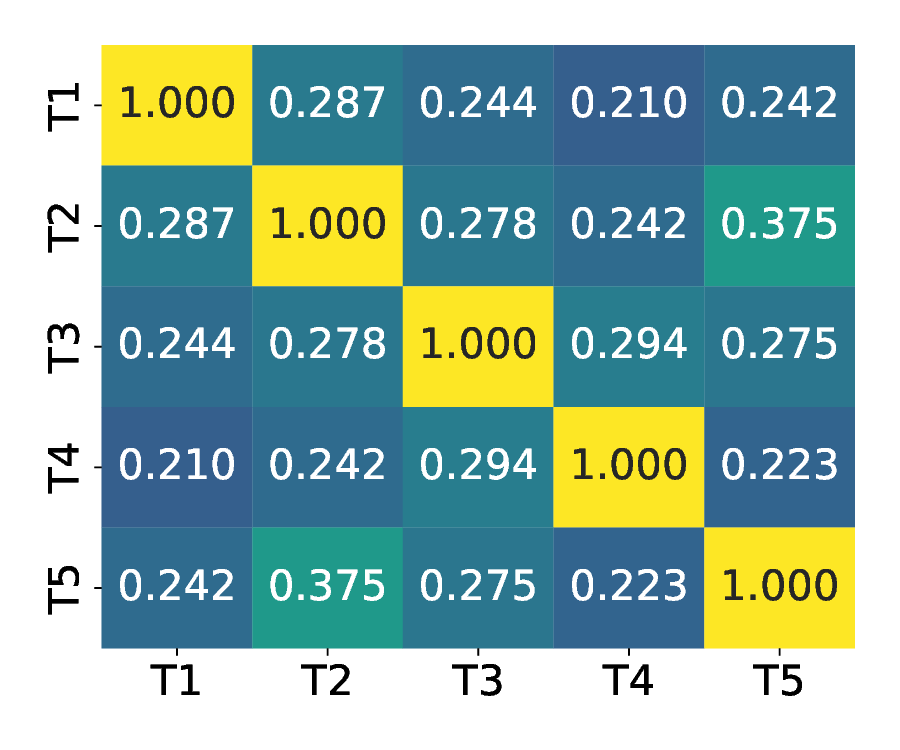}
		\end{minipage}
	}
    \subfigure[our method]{
		   \begin{minipage}[b]{0.2\textwidth}
		\includegraphics[width=1\textwidth]{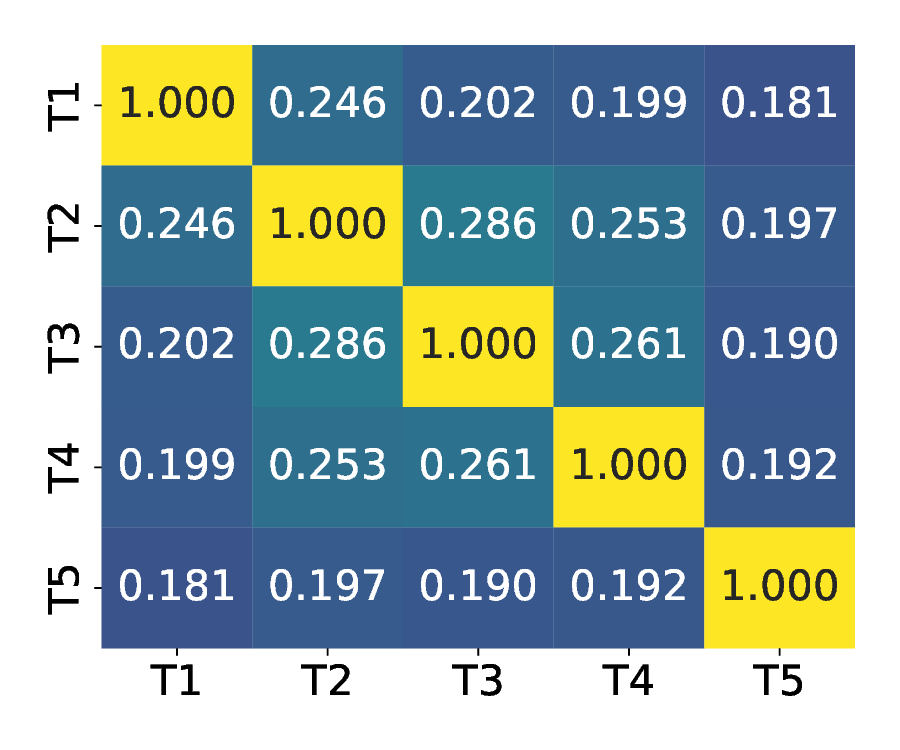}  
		\end{minipage}
	}
    \subfigure[oracle]{
		
        \begin{minipage}[b]{0.2\textwidth}
		\includegraphics[width=1\textwidth]{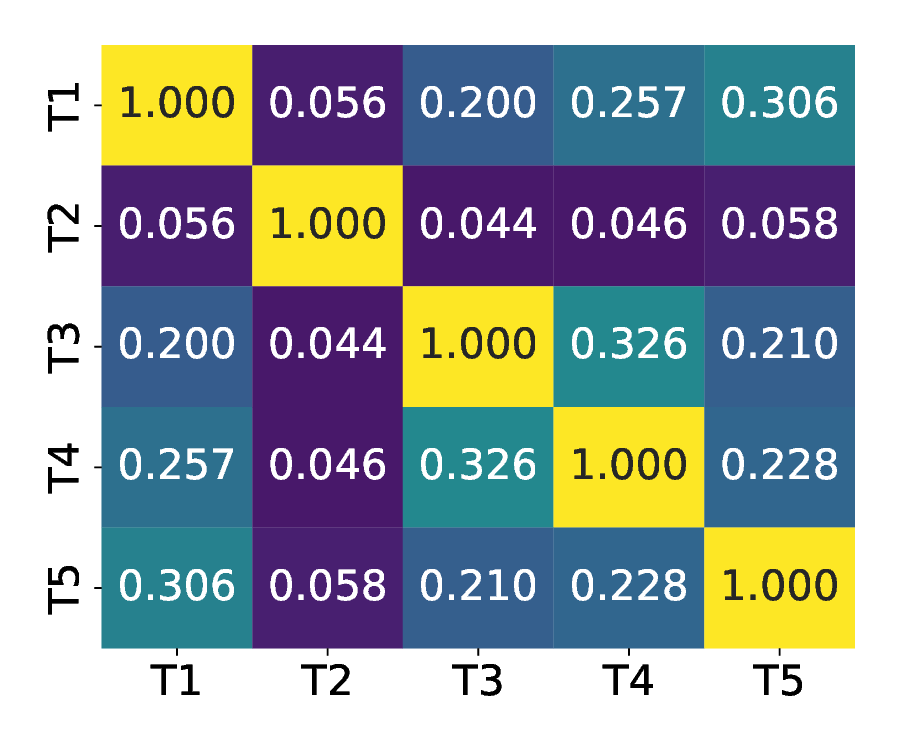}
		\end{minipage}
       
	}
		\label{fig:hor_2figs_1cap_2subcap_1}
	\caption{\textbf{Feature Overlap Heatmap}: darker colors represent less overlap, indicating better task separability. Each number represents the specific $E_O(P_A, P_B)$ between the corresponding row and column tasks.
    }
	\label{fig:resnet18_overlapping}
\end{figure*}

The results show that SupCon, with only 200 exemplars, exhibits significant overlap between tasks. Specifically, the overlap between tasks T4 and T3 reached 0.481. Co2l alleviates the feature overlap between tasks, but still shows a substantial gap compared to the oracle. However, our method shows the smallest difference from the oracle. This is because our approach constrains the features within a specific local region, effectively reducing the overlap between tasks.

\begin{figure*}[!htbp]
	\centering
            \begin{minipage}[t]{0.3\linewidth}
		\centering 
		\includegraphics[width=1.1\textwidth]{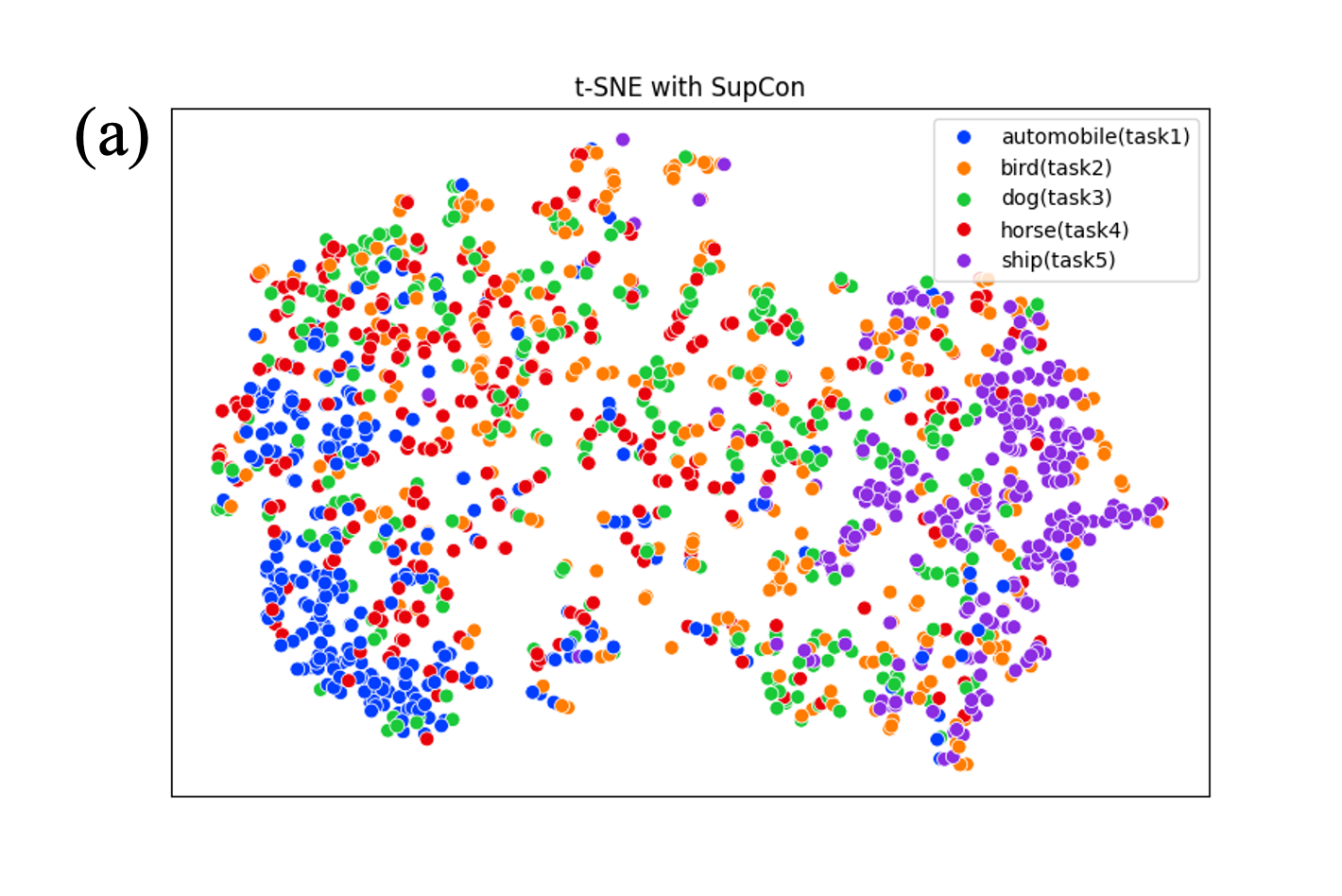}
	\end{minipage}
	\begin{minipage}[t]{0.3\linewidth}
		\centering
		\includegraphics[width=1.1\textwidth]{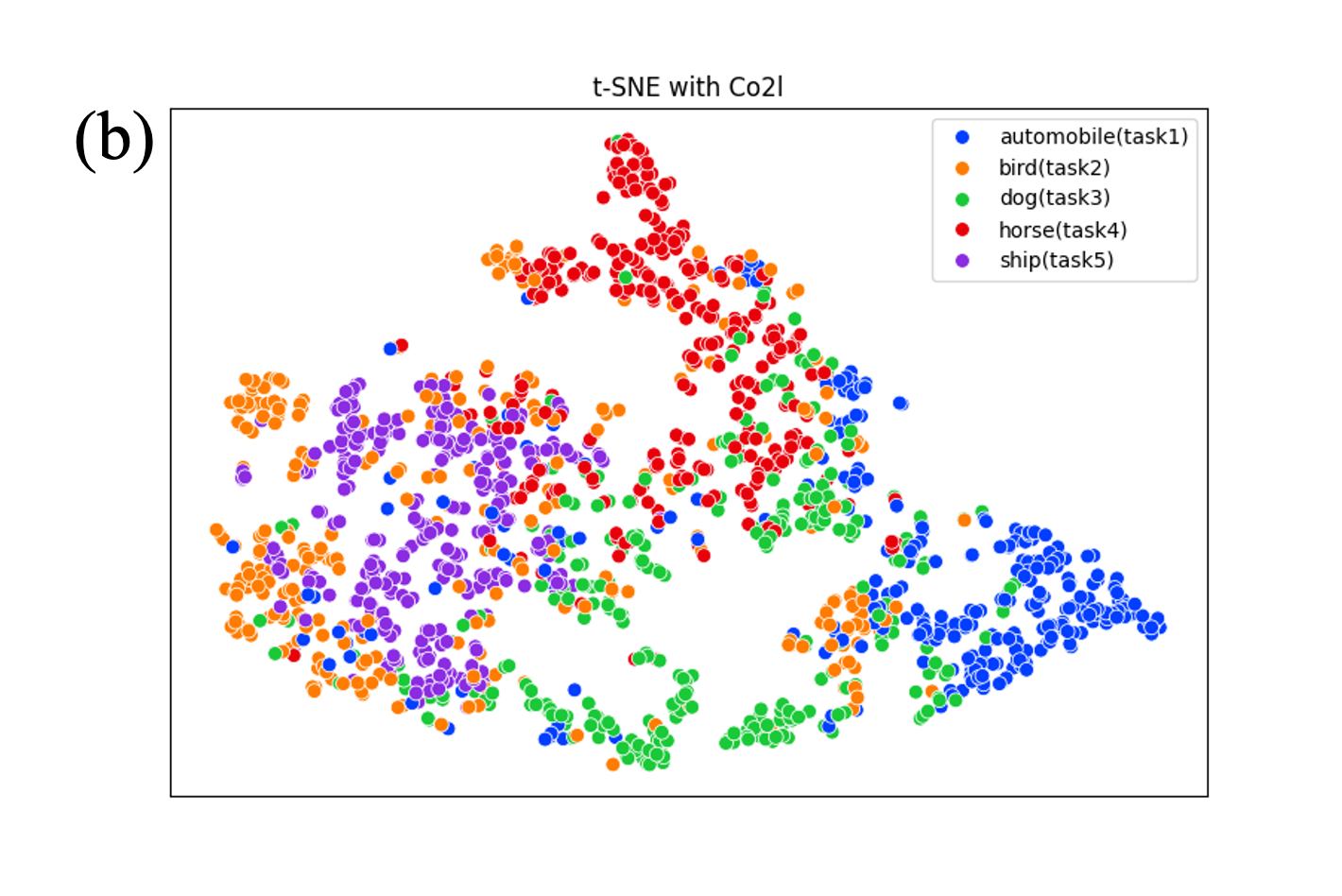}
	\end{minipage}
        \begin{minipage}[t]{0.3\linewidth}
		\centering
		\includegraphics[width=1.1\textwidth]{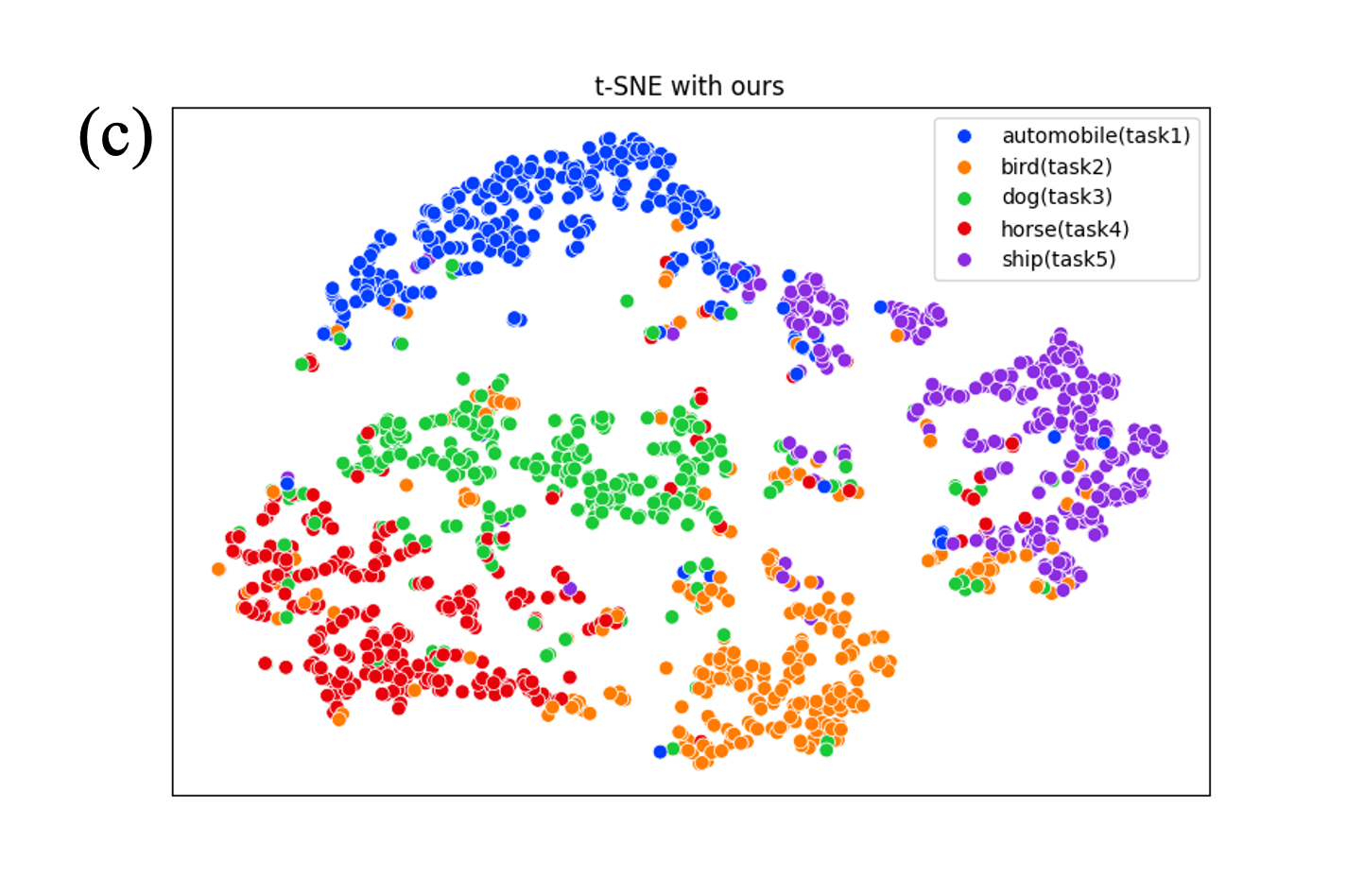}
	\end{minipage}
	\caption {t-SNE visualization of feature embeddings from validation set of Seq-Cifar-10. We select one class from each tasks. \textit{Features have been mapped onto the hypersphere}. (a) Fintune SupCon with 200 examplars. (b) Co2l, a contrastive continual learning algorithm with random sampling with 200 examplars. (c) Co2l with 200 examplars + our methods.}
	\label{fig:tsne}
\end{figure*}

\begin{table*}[H]
    \centering
    \begin{tabular}{cccccc}
        \hline  
        &\textbf{100}&\textbf{200}&\textbf{300}&\textbf{400}&\textbf{500}\\
        \hline
           SupCon &82.56$\pm$ 1.21&83.33$\pm$0.65&83.99$\pm$0.87&86.66$\pm$1.45&87.1$\pm$1.53\\
            Co2l &\textbf{93.11}$\pm$\textbf{1.65}&93.43$\pm$0.78&95.02$\pm $0.62&\textbf{95.90}$\pm$\textbf{0.54}&95.90$\pm$0.26\\
            \textbf{Co2l + ours} &92.33$\pm$1.23&\textbf{96.12}$\pm$\textbf{0.37}&\textbf{95.17}$\pm$\textbf{1.02}&95.07$\pm$0.48&\textbf{96.8}$\pm$\textbf{0.38}\\
        \hline

    \end{tabular}
    \caption{Different buffer sizes on Seq-Cifar-10 with Resnet18 (TIL).}
    \label{tab:buffer_abla_til}
\end{table*}

\begin{table*}[!htbp]
    \centering
    \label{tab:ablation_thresholg_loss}
    \begin{tabular}{lcc}
        \toprule
        \textbf{loss name} & \textbf{loss form} & Acc(CIL) \\
        \midrule
            Hinge  & $\mathcal{L}_{Hinge}(x)=\sum_{i \neq j}\max(0,x) $
                            & 70.09 \\
            Softplus  & $\mathcal{L}_{Softplus}(x) = 
                                       \log(1 + e^x)
                                        $  & 69.85 \\\rule{0pt}{3ex}
             LeakyReLU  &
                        $\mathcal{L}_{LeakyReLU}(x) = 
                                        \begin{cases} 
                                        x & \text{if } x > 0 \\
                                        \alpha x & \text{if } x \leq 0 
                                        \end{cases}
                                        $ 
            & 68.15 \\
            Squared Hinge  & 

            $\mathcal{L}_{Squared Hinge}(x) =\left( \max(0, x) \right)^2$
            & 66.82 \\
        \bottomrule
    \end{tabular}
   
     \caption{Ablation study of R2SCL and feature-level distillation. We train our model on the Seq-CIFAR-10 dataset with 200 examplars under a Class-IL scenario to explore the effectiveness of components.}
\end{table*}

\subsubsection{Different Buffer Sizes}
\label{sub:buffersize}
 We conduct additional experiments to explore the effectiveness of our algorithm with different buffer sizes. As illustrated in \Cref{tab:buffer_abla_cil}, with an increase in cached samples, the method achieves stable superiority over naive SupCon and Co2l within a specific range. However, as the cached sample quantity continues to grow, the advantage of our method gradually diminishes. The reason is that our method serves only as a complement to inter-task performance. As the number of samples increases, the improvement in inter-task discriminability diminishes.

\subsubsection{T-SNE: GPLASC Enhances Task-Level Separability}

\label{sub:tsne}
We conduct our ablation experiments with ResNet-18. We first select three baselines as shown in \Cref{fig:tsne}. We use Seq-CIFAR-10 as our training set. We collect the validation set and select one class from each task: \textit{automobile}, \textit{bird}, \textit{dog}, \textit{horse} and \textit{ship}. We plotted their t-SNE visualizations with the model after the last task. Note that the features are not mapped to the hypersphere. For the standard SupCon, which includes 200 examplars, there was almost no separability between tasks. By introducing IRD\cite{cha2021co2l}, the separability between tasks improves, but some classes, such as \textit{bird} and \textit{ship}, still overlap to some extent.

In contrast, our method achieved the best separability in t-SNE. Only a few points near the boundaries showed low separability. This improvement is primarily due to our method's ability to pre-separate subspaces of different tasks, naturally enhancing class separability across tasks.

\subsubsection{Extensive Sensitivity Analysis}
\label{sub:sensitivity}
\begin{figure*}[!h]
	\centering
    \includegraphics[width=0.8\textwidth]{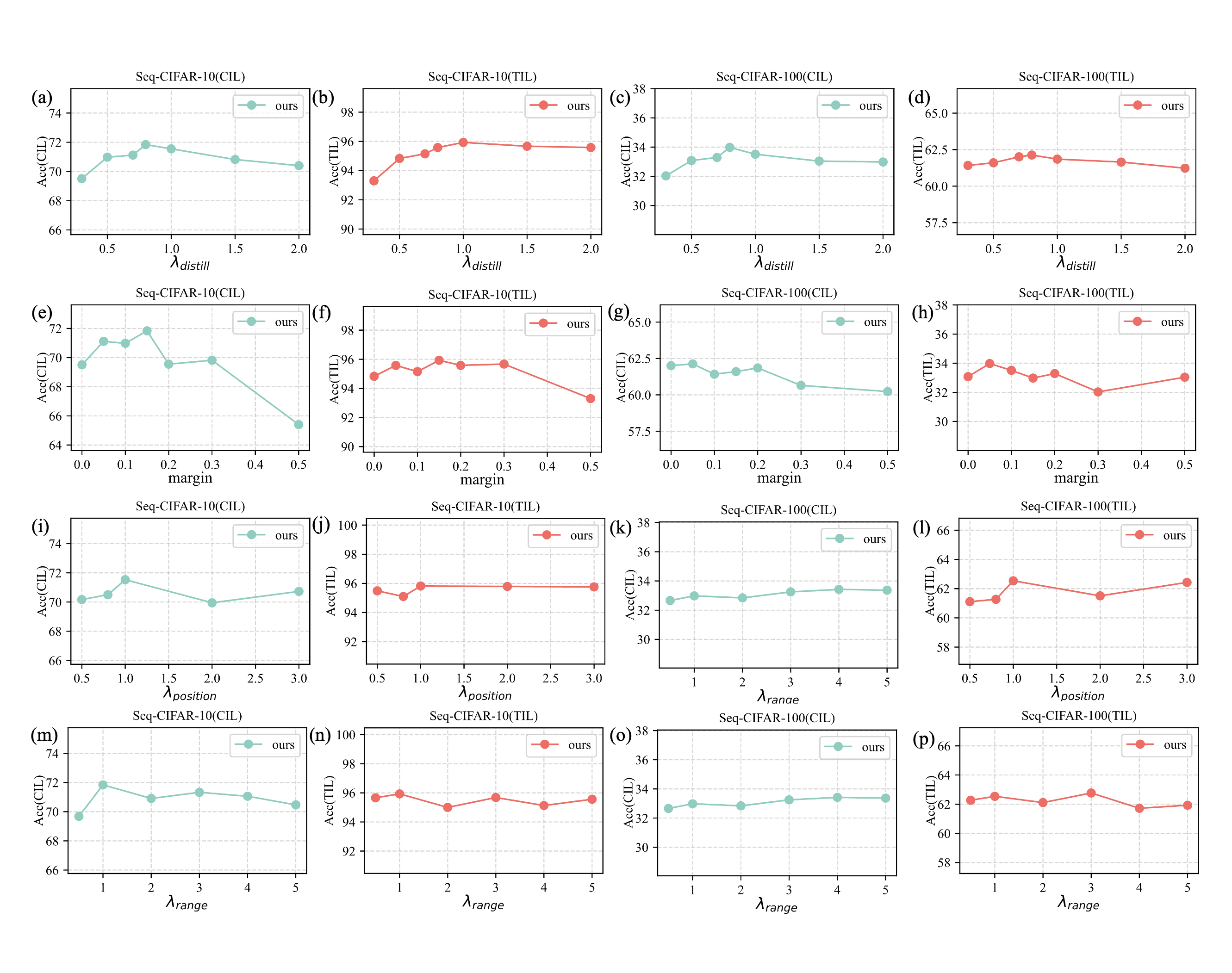}
	\caption{ Extensive sensitivity analysis was performed on five hyperparameters: $\lambda_{position}$, $\lambda_{distill}$, $\lambda_{range}$, \textit{margin} and pre-estimated task counts.
    (a-d) Sensitivity analysis is performed on $\lambda_{distill}$. 
    From left to right: CIL performance on cifar-10, TIL performance on cifar-10, CIL performance on cifar-100, TIL performance on cifar-100.
    (e-h) Sensitivity analysis is performed on \textit{margin}.
    (i-l) Sensitivity analysis is performed on $\lambda_{position}$. 
    (m-p) Sensitivity analysis is performed on $\lambda_{range}$.
    }

	\label{fig:all_ablation}
\end{figure*}

Our method includes four hyperparameters:$\lambda_{position}$,$ \lambda_{distill}$,
$\lambda_{range}$, \textit{margin}. To facilitate comparisons, we apply our methods to Co2l with 200 examplars. The results are shown in \Cref{fig:all_ablation}. Our approach consistently delivers robust outcomes across diverse hyperparameter values. Among the four hyperparameters, the \textit{margin} has a relatively greater impact on the results. The margin that is too small fails to ensure good inter-task separability, while the margin that is too large results in limited optimization regions within each task, thereby affecting the performance within tasks. Our empirical results show that a margin in the range of 0.1 to 0.15 leads to better performance in CIL.
Additional ablation studies also imply that our method can also consistently improve TIL performance.

\subsubsection{Effectiveness of Different Loss Function to Regulate Inter-Class Similarity}
\label{sec:threshold_loss}

We evaluate different loss function to regulate the threshold including \textit{Hinge }, \textit{Softplus }, \textit{LeakyReLU }, \textit{Squared Hinge }. 
Among them, the \textit{Softplus Loss} and \textit{Squared Hinge Loss} exhibit smoother gradient changes. For the \textit{LeakyReLU Loss}, we set the $\alpha$  to 0.1. As shown in \Cref{tab:ablation_thresholg_loss}, the Hinge Loss achieves the best performance.

\subsection{Analysis of Time Cost}
\label{sec:time}
A key consideration in our method's design is its computational efficiency. Our proposed framework integrates several components, including ETF-based classifier pre-allocation, supplementary loss terms (R2SCL, MSE), and a knowledge distillation mechanism. A natural concern arises regarding the potential increase in computational complexity introduced by these additions. While we have meticulously designed these components to be lightweight and operate within the existing network architecture without introducing additional forward or backward passes, we provide a quantitative analysis here to address this concern thoroughly.

To provide a concrete quantification of this overhead, we measured the total training time of our method and compared it against the original baseline methods on the CIFAR-10 dataset. The experiments were conducted on a single NVIDIA RTX 3090 GPU. The results are detailed in Table~\ref{tab:time_comparison}.
 
As shown in the table, our method introduces a modest computational overhead compared to the baselines across all tasks. For instance, in the first task ($t_1$), our method adds approximately 13.3\% and 9.4\% to the training time of Co2L and CCLIS, respectively. This increase is primarily attributed to the additional computations required by our framework, particularly the extensive inner product calculations used to measure the similarity between task-specific features and the pre-allocated task prototypes for the R2SCL loss.
\begin{table}[h!]
\centering
\caption{Total training time (in seconds) per task on CIFAR-10. We compare the original baseline methods (Co2L, CCLIS) with their enhanced versions incorporating our proposed modules ('+ours'). The time $t_i$ corresponds to the training duration for task $i$.}
\label{tab:time_comparison}
\begin{tabular}{lc} 
\hline\hline
\textbf{Methods}         & \textbf{Time (s)}  \\ 
\hline
Co2L ($t_1$)       & 1557               \\
\textbf{Co2L+ours ($t_1$)} & \textbf{1764}      \\
\hline
Co2L ($t_2$)       & 345                \\
\textbf{Co2L+ours ($t_2$)} & \textbf{444}       \\
\hline
Co2L ($t_3$)       & 371                \\
\textbf{Co2L+ours ($t_3$)} & \textbf{403}       \\
\hline
Co2L ($t_4$)       & 369                \\
\textbf{Co2L+ours ($t_4$)} & \textbf{424}       \\
\hline
Co2L ($t_5$)       & 354                \\
\textbf{Co2L+ours ($t_5$)} & \textbf{441}       \\
\hline\hline
CCLIS ($t_1$)      & 1991               \\
\textbf{CCLIS+ours ($t_1$)}& \textbf{2178}      \\
\hline
CCLIS ($t_2$)      & 410                \\
\textbf{CCLIS+ours ($t_2$)}& \textbf{454}       \\
\hline
CCLIS ($t_3$)      & 448                \\
\textbf{CCLIS+ours ($t_3$)}& \textbf{492}       \\
\hline
CCLIS ($t_4$)      & 413                \\
\textbf{CCLIS+ours ($t_4$)}& \textbf{497}       \\
\hline
CCLIS ($t_5$)      & 445                \\
\textbf{CCLIS+ours ($t_5$)}& \textbf{461}       \\
\hline\hline
\end{tabular}
\end{table}

\clearpage

\bibliography{ref}
\bibliographystyle{plain}

\end{document}